%% file: probabilistic_linear_solvers.tex
\documentclass{article}

\PassOptionsToPackage{numbers, compress}{natbib}

\usepackage[final]{neurips_2020}

\input{preamble_paper}
\input{math_commands.tex}

\title{Probabilistic Linear Solvers for Machine Learning}

\author{
	Jonathan Wenger \qquad Philipp Hennig\\
	University of T\" ubingen\\
	Max Planck Institute for Intelligent Systems\\T\" ubingen, Germany\\
	\texttt{\{jonathan.wenger, philipp.hennig\}@uni-tuebingen.de}
}

\begin{document}

\maketitle

\begin{abstract}
Linear systems are the bedrock of virtually all numerical computation. Machine learning poses specific challenges for the solution of such systems due to their scale, characteristic structure, stochasticity and the central role of uncertainty in the field. Unifying earlier work we propose a class of probabilistic linear solvers which jointly infer the matrix, its inverse and the solution from matrix-vector product observations. This class emerges from a fundamental set of desiderata which constrains the space of possible algorithms and recovers the method of conjugate gradients under certain conditions. We demonstrate how to incorporate prior spectral information in order to calibrate uncertainty and experimentally showcase the potential of such solvers for machine learning.
\end{abstract}

\section{Introduction}

Arguably one of the most fundamental problems in machine learning, statistics and scientific computation at large is the solution of linear systems of the form \(\mA \vx_* = \vb\), where \(\mA \in \mathbb{R}^{n \times n}_{\textup{sym}}\) is a symmetric positive definite matrix \cite{Saad1992, Trefethen1997, Golub2013}. Such matrices usually arise in the context of second-order or quadratic optimization problems and as Gram matrices. Some of the numerous application areas in machine learning and related fields are least-squares regression \cite{Bishop2006}, kernel methods \cite{Hofmann2008}, Kalman filtering \cite{Kalman1960}, Gaussian (process) inference \cite{Rasmussen2006}, spectral graph theory \cite{Chung1997}, (linear) differential equations \cite{Fletcher1984} and (stochastic) second-order methods \cite{Nocedal2006}.

Linear systems in machine learning are typically large-scale, have characteristic structure arising from generative processes, and are subject to noise. These distinctive features call for linear solvers that can explicitly make use of such structural information. While classic solvers are highly optimized for general problems, they lack key functionality for machine learning. In particular, they do not consider generative prior information about the matrix.

An important example are kernel Gram matrices, which exhibit specific sparsity structure and spectral properties, depending on the kernel choice and the generative process of the data. Exploiting such prior information is a prime application for probabilistic linear solvers, which aim to quantify numerical uncertainty arising from limited computational resources. Another key challenge, which we will not yet address here, are noisy matrix evaluations arising from data subsampling. Ultimately, linear algebra for machine learning should integrate all sources of uncertainty in a computational pipeline -- aleatoric, epistemic and numerical -- into one coherent probabilistic framework.

\begin{figure}
        \centering
        \includegraphics[width=\textwidth]{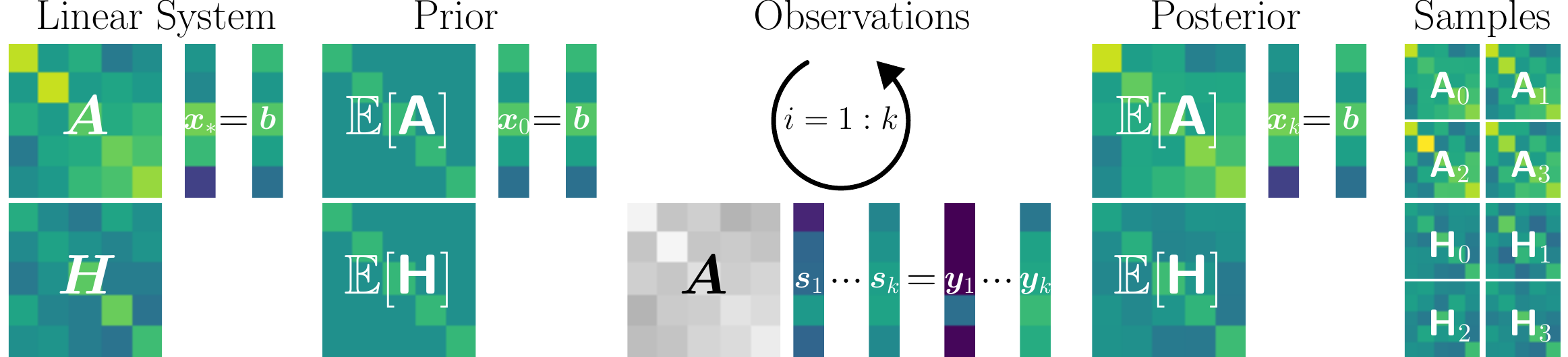}
	\caption{\textit{Illustration of a probabilistic linear solver.} Given a prior for \(\rmA\) or \(\rmH\) modelling the linear operator \(\mA\) and its inverse \(\mA^{-1}\), posterior beliefs are inferred via observations \(\vy_i = \mA\vs_i\). This induces a distribution on the solution \(\vx_*\), quantifying numerical uncertainty arising from finite computation. The plot shows \(k=3\) iterations of \Cref{alg:problinsolve} on a toy problem of dimension \(n=5\).\label{fig:problinear_solver}}
\end{figure}

\paragraph{Contribution}
This paper sets forth desiderata for probabilistic linear solvers which establish first principles for such methods. From these, we derive an algorithm incorporating prior information on the matrix \(\mA\) or its inverse \(\mA^{-1}\), which jointly estimates both via repeated application of \(\mA\). This results in posterior beliefs over the two operators and the solution which quantify numerical uncertainty. Our approach unifies and extends earlier formulations and constitutes a new way of interpreting linear solvers. Further, we propose a prior covariance class which recovers the method of conjugate gradients as its posterior mean and uses prior spectral information for uncertainty calibration, one of the primary shortcomings of  probabilistic linear solvers. We conclude by presenting simplified examples of promising applications of such solvers within machine learning.

\section{Probabilistic Linear Solvers}
\label{sec:probabilistic_linear_solvers}

Let \(\mA \vx_* = \vb\) be a linear system with \(\mA \in \mathbb{R}^{n \times n}_{\textup{sym}}\) positive definite and \(\vb \in \R^n\). \emph{Probabilistic linear solvers} (PLS) \cite{Hennig2015, Cockayne2019a,Bartels2019} iteratively build a model for the linear operator \(\mA\), its inverse \(\mH = \mA^{-1}\) or the solution \(\vx_*\), represented by random variables \(\rmA, \rmH\) or \(\rvx\). In the framework of probabilistic numerics  \cite{Hennig2015a, Oates2019} such solvers can be seen as Bayesian agents performing \emph{inference} via linear \emph{observations} \(\mY = [\vy_1, \dots, \vy_k] \in \mathbb{R}^{n \times k}\) resulting from \emph{actions} \(\mS = [\vs_1, \dots, \vs_k]  \in \mathbb{R}^{n \times k}\) given by an internal \emph{policy} \(\pi(\vs \mid \rmA, \rmH, \rvx, \mA, \vb)\). For a matrix-variate prior \(p(\rmA)\) or \(p(\rmH)\) encoding prior (generative) information, our solver computes posterior beliefs over the matrix, its inverse and the solution of the linear system. An illustration of a probabilistic linear solver is given in \Cref{fig:problinear_solver}.

\paragraph{Desiderata}
We begin by stipulating a fundamental set of desiderata for probabilistic linear solvers. To our knowledge such a list has not been collated before. Connecting previously disjoint threads, the following presents a roadmap for the development of these methods. Probabilistic linear solvers modelling \(\mA\) and \(\mA^{-1}\) must assume matrix-variate distributions which are expressive enough to capture structure and generative prior information either for \(\mA\) or its inverse. The distribution choice must also allow computationally efficient sampling and density evaluation. It should encode symmetry and positive definiteness and must be closed under positive linear combinations. Further, the two models for the system matrix or its inverse should be translatable into and consistent with each other. Actions \(\vs_i\) of a PLS should be model-based and induce a tractable distribution on linear observations \(\vy_i = \mA \vs_i\). Since probabilistic linear solvers are low-level procedures, their inference procedure must be computationally lightweight. Given (noise-corrupted) observations this requires tractable posteriors over \(\rmA\), \(\rmH\) and \(\rvx\), which are calibrated in the sense that at convergence the true solution \(\vx_*\) represents a draw from the posterior \(p(\rvx \mid \mY, \mS)\). Finally, such solvers need to allow preconditioning of the problem and ideally should return beliefs over non-linear properties of the system matrix extending the functionality of classic methods. These desiderata are summarized concisely in \Cref{tab:desiderata}.

\begin{table}
\caption{\textit{Desired properties of probabilistic linear solvers.} Symbols (\xmark , \umark , \cmark ) indicate which properties are encoded in our proposed solver (see \Cref{alg:problinsolve}) and to what degree.\label{tab:desiderata}}
\vspace{-.5em}
\begin{center}
\begin{small}
    \begin{tabular}{rlcc}
        \toprule
        No. & Property & Formulation &  \\
        \midrule
        (1) & distribution over matrices & \(\rmA \sim \mathcal{D}, \ p_{\mathcal{D}}(\rmA)\)& \cmark \\
        (2) & symmetry & \(\rmA = \rmA^{\top}\) a.s. & \cmark \\
        (3) & positive definiteness & \(\forall \vv \neq 0 : \ \vv^{\top}\rmA \vv > 0 \ \) a.s.& \umark \\
        (4) & positive linear combination in same distribution family& \(\forall \alpha_j > 0 : \ \sum_j \alpha_j \rmA_j \sim \mathcal{D}\)& \cmark \\
        (5) & corresponding priors on the matrix and its inverse & \(p(\rmA) \longleftrightarrow p(\rmH)\)& \cmark \\
        \midrule
        (6) & model-based policy & \(\vs_i \sim \pi( \vs \mid \mA, \vb, \rmA, \rmH, \rvx)\)& \cmark \\
        (7) & matrix-vector product in tractable distribution family & \(\rmA\vs \sim \mathcal{D}'\)& \cmark \\
        (8) & noisy observations & \(p(\mY \mid \rmA, \mS) = \mathcal{N}(\mY; \mA\mS, \bm{\Lambda})\)& \xmark\\
        (9) & tractable posterior & \(p(\rmA \mid \mY, \mS)\) or \(p(\rmH \mid \mY, \mS)\)& \cmark \\
       (10) & calibrated uncertainty & \(\vx_* \sim \mathcal{N}(\mathbb{E}[\rvx], \operatorname{Cov}[\rvx])\) & \umark \\
        \midrule
        (11) & preconditioning & \((\mP^{-\top}\mA \mP^{-1}) \mP \vx_* = \mP^{-\top}\vb\)& \cmark \\
        (12) & distributions over non-linear derived quantities of \(\mA\)& \(\operatorname{det}(\rmA), \, \sigma(\rmA), \, \rmA = \rmL^\top \rmL, \dots\) & \xmark \\
        \bottomrule
    \end{tabular}
\end{small}
\end{center}
\vspace{-.5em}
\end{table}

\subsection{Bayesian Inference Framework}
\label{sec:inference_framework}

Guided by these desiderata, we will now outline the inference framework for \(\rmA, \rmH\) and \(\rvx\) forming the base of the algorithm. The choice of a matrix-variate prior distribution is severely limited by the desideratum that conditioning on linear observations \(\vy_i = \mA \vs_i\) must be tractable. This reduces the choice to stable distributions \cite{Levy1925} and thus excludes candidates such as the Wishart, which has measure zero outside the cone of symmetric positive semi-definite matrices. For symmetric matrices, this essentially forces use of the symmetric matrix-variate normal distribution, introduced in this context by \citet{Hennig2015}. Given \(\mA_0, \mW_0^{\rmA} \in \mathbb{R}^{n \times n}_{\textup{sym}}\), assume a prior distribution
\begin{equation*}
p(\rmA) = \mathcal{N}(\rmA; \mA_0, \mW_0^{\rmA} \ostimes \mW_0^{\rmA}),
\end{equation*}
where \(\ostimes\) denotes the symmetric Kronecker product \cite{Loan2000}.\footnote{See \Cref{sec:kronecker_products,sec:matrixvariate_normal} of the supplementary material for more detail on Kronecker-type products and matrix-variate normal distributions.} The symmetric matrix-variate Gaussian induces a Gaussian distribution on linear observations. While it has non-zero measure only for symmetric matrices, its support is not the positive definite cone. However, positive definiteness can still be enforced post-hoc (see \Cref{thm:hereditary_posdef}). We assume noise-free linear observations of the form \(\vy_i = \mA \vs_i\), leading to a Dirac likelihood
\begin{equation*}
p(\mY \mid \rmA, \mS) = \lim_{\varepsilon \downarrow 0} \mathcal{N}(\mY; \mA \mS, \varepsilon^2 \mI \otimes \mI)=\delta(\mY - \mA \mS).
\end{equation*}
The posterior distribution follows from the properties of Gaussians \cite{Bishop2006} and has been investigated in detail in previous work \cite{Hennig2013, Hennig2015, Bartels2019}. It is given by \(p(\rmA \mid \mS, \mY)  = \mathcal{N}(\rmA; \mA_k, \mSigma_k)\) with
\begin{align*}
\mA_k&= \mA_0 + \mDelta_0^{\rmA} \mU^\top + \mU (\mDelta_0^{\rmA})^\top - \mU \mS^\top \mDelta_0^{\rmA} \mU^\top \\
\mSigma_k &=\mW_0^{\rmA} (\mI_n-\mS\mU^\top) \ostimes \mW_0^{\rmA} (\mI_n-\mS\mU^\top)
\end{align*}
where \(\mDelta_0^{\rmA} = \mY - \mA_0\mS\) and \(\mU = \mW_0^{\rmA} \mS (\mS^\top \mW_0^{\rmA} \mS)^{-1}\). We aim to construct a probabilistic model \(\rmH\) for the inverse \(\mH=\mA^{-1}\) consistent with the model \(\rmA\) as well. However, not even in the scalar case does the inverse of a Gaussian have finite mean. We ask instead what Gaussian model for \(\rmH\) is as consistent as possible with our observational model for \(\rmA\). For a prior of the form \(p(\rmH) = \mathcal{N}(\rmH; \mH_0, \mW_0^{\rmH} \ostimes \mW_0^{\rmH})\) and likelihood \(p(\mS \mid \rmH, \mY) = \delta(\mS - \mH \mY)\), we analogously to the \(\rmA\)-model obtain a posterior distribution \(p(\rmH \mid \mS, \mY)  = \mathcal{N}(\rmH; \mH_k, \mSigma^{\rmH}_k)\) with
\begin{align*}
\mH_k&= \mH_0 +\mDelta_0^{\rmH}(\mU^{\rmH})^\top + \mU^{\rmH}(\mDelta_0^{\rmH})^\top - \mU^{\rmH} \mY^\top \mDelta_0^{\rmH} (\mU^{\rmH})^\top \\
\mSigma^{\rmH}_k &=\mW_0^{\rmH} (\mI_n-\mY(\mU^{\rmH})^\top) \ostimes \mW_0^{\rmH} (\mI_n-\mY(\mU^{\rmH})^\top)
\end{align*}
where \(\mDelta_0^{\rmH} = \mS - \mH_0\mY\) and \(\mU^{\rmH} = \mW_0^{\rmH} \mY (\mY^\top \mW_0^{\rmH} \mY)^{-1}\). In \Cref{sec:prior_covariance_class} we will derive a covariance class, which establishes correspondence between the two Gaussian viewpoints for the linear operator and its inverse and is consistent with our desiderata.

\subsection{Algorithm}
\label{sec:algorithm}
The above inference procedure leads to \Cref{alg:problinsolve}. The degree to which the desiderata are encoded in our formulation of a PLS can be found in \Cref{tab:desiderata}. We will now go into more detail about the policy, the choice of step size, stopping criteria and the implementation.

\begin{algorithm}
	\caption{Probabilistic Linear Solver with Uncertainty Calibration}
	\label{alg:problinsolve}
	\begin{algorithmic}[1]
	\Procedure{\textsc{ProbLinSolve}}{$\mA(\cdot), \vb, \rmA, \rmH$} 															\Comment{prior for $\rmA$ or $\rmH$}
		\State \(\vx_0 \gets \mathbb{E}[\rmH]\vb\) 																			\Comment{initial guess}
		\State \(\phantom{\vx_0}\mathllap{\vr_0} \gets \mA\vx_0 - \vb\)
		\While{$\min(\sqrt{\tr(\operatorname{Cov}[\rvx])}, \norm{\vr_i}_2) > \max(\delta_{\textup{rtol}} \norm{\vb}_2, \delta_{\textup{atol}})$ 
		}	\Comment{stopping criteria}
			\State \({\vs_i} \gets - \mathbb{E}[\rmH]\vr_{i-1}\)
															\Comment{compute action via policy}
			\State \(\phantom{\vs_i}\mathllap{\vy_i} \gets \mA \vs_i\) 																			\Comment{make observation}
			\State \(\phantom{\vs_i}\mathllap{\alpha_i} \gets - \vs_i^{\top} \vr_{i-1}(\vs_i^\top \vy_i)^{-1}\) 												\Comment{optimal step size}
			\State \(\phantom{\vs_i}\mathllap{\vx_{i}} \gets \vx_{i-1} + \alpha_i \vs_i\) 																\Comment{update solution estimate}
			\State \(\phantom{\vs_i}\mathllap{\vr_i} \gets \vr_{i-1} + \alpha_i \vy_i\) 																\Comment{update residual}
			\State \(\phantom{\vs_i}\mathllap{\rmA} \gets \textsc{Infer}(\rmA, \vs_i, \vy_i)\) 															\Comment{infer posterior distributions}
			\State \(\phantom{\vs_i}\mathllap{\rmH} \gets \textsc{Infer}(\rmH, \vs_i, \vy_i)\) 															\Comment{(see \Cref{sec:inference_framework})}
	    	\State \(\mPhi, \mPsi \gets \textsc{Calibrate}(\mS, \mY)\) \Comment{calibrate uncertainty}
	    \EndWhile
	    \State \(\rvx \gets \mathcal{N}(\vx_k, \operatorname{Cov}[\rmH \vb])\) \Comment{belief over solution}
    	\State \Return $(\rvx, \rmA, \rmH)$
	\EndProcedure
	\end{algorithmic}
\end{algorithm}

\paragraph{Policy and Step Size}
In each iteration our solver collects information about the linear operator \(\mA\) via actions \(\vs_i\) determined by the policy \(\pi(\vs \mid \rmA, \rmH, \rvx, \mA, \vb)\). The next action \(\vs_i = - \mathbb{E}[\rmH]\vr_{i-1}\) is chosen based on the current belief about the inverse. If \(\mathbb{E}[\rmH]= \mA^{-1}\), i.e. if the solver's estimate for the inverse equals the true inverse, then \Cref{alg:problinsolve} converges in a single step since
\begin{equation*}
\vx_{i-1} + \vs_i =\vx_{i-1} - \mathbb{E}[\rmH] \vr_{i-1}=\vx_{i-1}  -\mA^{-1}(\mA\vx_{i-1} - \vb) = \mA^{-1}\vb = \vx_*.
\end{equation*}
The step size minimizing the quadratic \(q(\vx_i + \alpha \vs_i) = \frac{1}{2}(\vx_i + \alpha \vs_i)^\top \mA (\vx_i + \alpha \vs_i) - \vb^\top(\vx_i + \alpha \vs_i)\) along the action \(\vs_i\) is given by \(\alpha_i = \argmin_\alpha q(\vx_i + \alpha \vs_i) = \vs_i^\top(\vb - \mA \vx_i)(\vs_i^\top \mA \vs_i)^{-1}\).

\paragraph{Stopping Criteria}
Classic linear solvers typically use stopping criteria based on the current residual of the form \(\norm{\mA \vx_i - \vb}_2 \leq \max(\delta_{\textup{rtol}} \norm{\vb}_2, \delta_{\textup{atol}})\) for relative and absolute tolerances \(\delta_{\textup{rtol}}\) and \(\delta_{\textup{atol}}\). However, this residual may oscillate or even increase in all but the last step even if the error \(\norm{\vx_* - 
\vx_i}_2\) is monotonically decreasing \cite{Hestenes1952, Gutknecht2000}. From a probabilistic point of view, we should stop if our posterior uncertainty is sufficiently small. Assuming the posterior covariance is calibrated, it holds that \((\mathbb{E}_{\vx_*}[\norm{\vx_* - \mathbb{E}[\rvx]}_2])^2 \leq \mathbb{E}_{\vx_*}[\norm{\vx_* - \mathbb{E}[\rvx]}_2^2] = \tr(\operatorname{Cov}[\rvx])\). Hence given calibration, we can bound the expected (relative) error between our estimate and the true solution by terminating when \(\sqrt{\tr(\operatorname{Cov}[\rvx])} \leq \max(\delta_{\textup{rtol}} \norm{\vb}_2, \delta_{\textup{atol}})\). A probabilistic criterion is also necessary for an extension to the noisy setting, where classic convergence criteria become stochastic. However, probabilistic linear solvers typically suffer from miscalibration \cite{Cockayne2019}, an issue we will address in \Cref{sec:prior_covariance_class}.

\paragraph{Implementation}
We provide an open-source implementation of \Cref{alg:problinsolve} as part of \textsc{ProbNum}, a Python package implementing probabilistic numerical methods, in an online code repository:
\begin{center}
\vspace{-0.5em}
\begin{tabular}{m{1cm} l}
\includegraphics[width=0.08\textwidth]{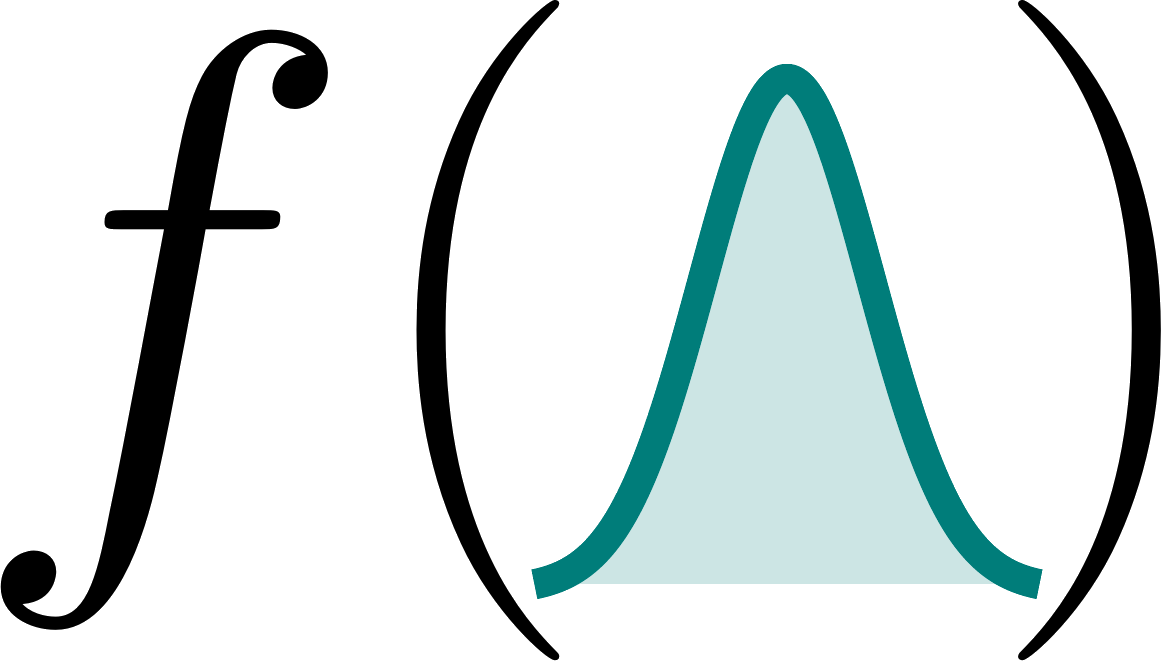} & \url{https://github.com/probabilistic-numerics/probnum}\\
\end{tabular}
\vspace{-0.5em}
\end{center}
The mean and covariance up- and downdates in \Cref{sec:inference_framework} when performed iteratively are of low rank. In order to maintain numerical stability these updates can instead be performed for their respective Cholesky factors \cite{Seeger2008}. This also enables computationally efficient sampling or evaluation of probability density functions downstream.

\subsection{Theoretical Properties}
\label{sec:theoretical_properties}
This section details some theoretical properties of our method such as its convergence behavior and computational complexity. In particular we demonstrate that for a specific prior choice \Cref{alg:problinsolve} recovers the method of conjugate gradients as its solution estimate. All proofs of results in this section and the next can be found in the supplementary material. We begin by establishing that our solver is a \emph{conjugate directions method} and therefore converges in at most \(n\) steps in exact arithmetic.

\begin{theorem}[Conjugate Directions Method]
\label{thm:conj_directions_method}
Given a prior \(p(\rmH)=\mathcal{N}(\rmH; \mH_0, \mW_0^{\rmH} \ostimes \mW_0^{\rmH})\) such that \(\mH_0, \mW_0^{\rmH} \in \mathbb{R}^{n \times n}_{\textup{sym}}\) positive definite, then actions \(\vs_i\) of \Cref{alg:problinsolve} are \(\mA\)-conjugate, i.e. for \(0 \leq i,j \leq k\) with \(i \neq j\) it holds that \(\vs_i^\top \mA \vs_j = 0\).
\end{theorem}
We can obtain a better convergence rate by placing stronger conditions on the prior covariance class as outlined in \Cref{sec:prior_covariance_class}. Given these assumptions, \Cref{alg:problinsolve} recovers the iterates of (preconditioned) CG and thus inherits its favorable convergence behavior (overviews in \cite{Luenberger1973, Nocedal2006}).

\begin{theorem}[Connection to the Conjugate Gradient Method]
\label{thm:connection_cg}
Given a scalar prior mean \(\mA_0 = \mH_0^{-1}= \alpha \mI\) with \(\alpha > 0\), assume \eqref{eqn:hered_pos_def} and \eqref{eqn:post_mean_equiv} hold, then the iterates \(\bm{x}_i\) of \Cref{alg:problinsolve} are identical to the ones produced by the conjugate gradient method.
\end{theorem}

A common phenomenon observed when implementing conjugate gradient methods is that due to cancellation in the computation of the residuals, the search directions \(\vs_i\) lose \(\mA\)-conjugacy \citep{Paige1972, Simon1984, Golub2013}. In fact, they can become independent up to working precision for \(i\) large enough \citep{Simon1984}. One way to combat this is to perform complete reorthogonalization of the search directions in each iteration as originally suggested by \citet{Lanczos1950}. \Cref{alg:problinsolve} does this \emph{implicitly} via its choice of policy which depends on all previous search directions as opposed to just \(\vs_{i-1}\) for (naive) CG.

\paragraph{Computational Complexity}
The solver has time complexity \(\mathcal{O}(kn^2)\) for \(k\) iterations without uncertainty calibration. Compared to CG, inferring the posteriors in \Cref{sec:inference_framework} adds an overhead of four outer products and four matrix-vector products per iteration, given \eqref{eqn:hered_pos_def} and \eqref{eqn:post_mean_equiv}. Uncertainty calibration outlined in \Cref{sec:prior_covariance_class} adds between \(\mathcal{O}(1)\) and \(\mathcal{O}(k^3)\) per iteration depending on the sophistication of the scheme. Already for moderate \(n\) this is dominated by the iteration cost. In practice, means and covariances do not need to be formed in memory. Instead they can be evaluated lazily as linear operators \(\vv \mapsto \mL \vv\), if \(\mS\) and \(\mY\) are stored. This results in space complexity \(\mathcal{O}(kn)\).

\subsection{Related Work}
Numerical methods for the solution of linear systems have been studied in great detail since the last century. Standard texts \cite{Saad1992, Trefethen1997, Nocedal2006, Golub2013} give an in-depth overview. The conjugate gradient method  recovered by our algorithm for a specific choice of prior was introduced by \citet{Hestenes1952}.
Recently, randomization has been exploited to develop improved algorithms for large-scale problems arising from machine learning \citep{Drineas2016, Gittens2016}. The key difference to our approach is that we do not rely on sampling to approximate large-scale matrices, but instead perform probabilistic inference.
Our approach is based on the framework of probabilistic numerics \cite{Hennig2015a, Oates2019} and
is a natural continuation of previous work on probabilistic linear solvers. In historical order, \citet{Hennig2013} provided a probabilistic interpretation of Quasi-Newton methods, which was expanded upon in \cite{Hennig2015}. This work also relied on the symmetric matrix-variate Gaussian as used in our paper. \citet{Bartels2016} estimate numerical error in approximate least-squares solutions by using a probabilistic model. More recently, \citet{Cockayne2019} proposed a Bayesian conjugate gradient method performing inference on the solution of the system. This was connected to the matrix-based view by \citet{Bartels2019}.

\section{Prior Covariance Class}
\label{sec:prior_covariance_class}

Having outlined the proposed algorithm, this section derives a prior covariance class which satisfies nearly all desiderata, connects the two modes of prior information and allows for calibration of uncertainty by appropriately choosing remaining degrees of freedom in the covariance. The third desideratum posited that \(\rmA\) and \(\rmH\) should be almost surely positive definite. This evidently does not hold for the matrix-variate Gaussian. However, we can restrict the choice of admissable \(\mW_0^\rmA\) to act like \(\mA\) on \(\operatorname{span}(\mS)\). This in turn induces a positive definite posterior mean.

\begin{proposition}[Hereditary Positive Definiteness \cite{Dennis1977, Hennig2013}]
\label{thm:hereditary_posdef}
Let \(\mA_0 \in \mathbb{R}^{n \times n}_{\textup{sym}}\) be positive definite. Assume the actions \(\mS\) are \(\mA\)-conjugate and \(\mW_0^\rmA \mS = \mY\), then for \(i \in \{0, \dots, k - 1\}\) it holds that \(\mA_{i+1}\) is symmetric positive definite.
\end{proposition}

Prior information about the linear system usually concerns the matrix \(\mA\) itself and not its inverse, but the inverse is needed to infer the solution $\vx_*$ of the linear problem. So a way to translate between a Gaussian distribution on \(\rmA\) and \(\rmH\) is crucial. Previous works generally committed to either one view or the other, potentially discarding available information. Below, we show that the two correspond, if we allow ourselves to constrain the space of possible models. We impose the following condition.

\begin{definition}
Let \(\mA_i\) and \(\mH_i\) be the means of \(\rmA\) and \(\rmH\) at step \(i\). We say a prior induces \emph{posterior correspondence} if
\(\mA_i^{-1} = \mH_i\) for all \(0 \leq i \leq k\). If only \(\mA_i^{-1}\mY = \mH_i \mY,\) \emph{weak posterior correspondence} holds.
\end{definition}
The following theorem establishes a sufficient condition for weak posterior correspondence. For an asymmetric prior model one can establish the stronger notion of posterior correspondence. A proof is included in the supplements.

\begin{theorem}[Weak Posterior Correspondence]
\label{thm:weak_post_correspondence}
Let \(\mW_0^\rmH \in \mathbb{R}^{n \times n}_{\textup{sym}}\) be positive definite. Assume \(\mH_0 = \mA_0^{-1}\), and that $\mW_0^\rmA, \mA_0, \mW_0^\rmH$ satisfy
\begin{align}
\label{eqn:hered_pos_def}
\mW_0^\rmA \mS &=\mY,\\
\label{eqn:post_mean_equiv}
\mS^\top (\mW_0^\rmA \mA_0^{-1} - \mA\mW_0^\rmH) &= \bm{0},
\end{align}
then weak posterior correspondence holds for the symmetric Kronecker covariance.
\end{theorem}

Given the above, let \(\mA_0\) be a symmetric positive definite prior mean and \(\mH_0 = \mA_0^{-1}\). Define the orthogonal projection matrices \(\mP_{\mS^\perp} = \mI-\mS( \mS^\top \mS)^{-1}\mS^\top \in \mathbb{R}^{n \times n}_{\text{sym}}\) and \(\mP_{\mY^\perp} = \mI - \mY(\mY^\top \mY)^{-1}\mY^\top  \in \mathbb{R}^{n \times n}_{\text{sym}}\) mapping to the spaces \(\operatorname{span}(\mS)^\perp\) and \(\operatorname{span}(\mY)^\perp\). We propose the following prior covariance class given by the prior covariance factors of the \(\rmA\) and \(\rmH\) view
\begin{equation}
\label{eqn:prior_cov_class}
\begin{aligned}
\mW_0^\rmA &= \mA \mS(\mS^\top \mA \mS)^{-1} \mS^\top \mA +\mP_{\mS^\perp}\mPhi \mP_{\mS^\perp},\\
\mW_0^\rmH &= \mA_0^{-1}\mY(\mY^\top \mA_0^{-1} \mY)^{-1} \mY^\top \mA_0^{-1} + \mP_{\mY^\perp} \mPsi \mP_{\mY^\perp},
\end{aligned}
\end{equation}
where \(\mPhi \in \mathbb{R}^{n \times n}\) and \(\mPsi \in \mathbb{R}^{n \times n}\) are degrees of freedom. This choice of covariance class satisfies \Cref{thm:conj_directions_method}, \Cref{thm:hereditary_posdef}, \Cref{thm:weak_post_correspondence} and for a scalar mean also \Cref{thm:connection_cg}. Therefore, it produces symmetric realizations, has symmetric positive semi-definite means, it links the matrix and the inverse view and at any given time only needs access to \(\vv \mapsto \mA \vv\) not \(\mA\) itself. It is also compatible with a preconditioner by simply transforming the given linear problem. 

This class can be interpreted as follows. The derived covariance factor \(\mW_0^\rmA\) acts like \(\mA\) on the space \(\operatorname{span}(\mS)\) explored by the algorithm. On the remaining space its uncertainty is determined by the degrees of freedom in \(\mPhi\). Likewise, our best guess for \(\mA^{-1}\) is \(\mA_0^{-1}\) on the space spanned by \(\mY\). On the orthogonal space \(\operatorname{span}(\mY)^\perp\) the uncertainty is determined by \(\mPsi\). Note that the prior depends on actions and observations collected during a run of \Cref{alg:problinsolve}, hence one might call this an empirical Bayesian approach. This begs the question how the algorithm is realizable for the proposed prior \eqref{eqn:prior_cov_class} given its dependence on future data. Notice that the posterior mean in \Cref{sec:inference_framework} only depends on \(\mW_0^\rmA \mS = \mY\) \emph{not} on \(\mW_0^\rmA\) alone. Using \cref{eqn:prior_cov_class}, at iteration \(i\) we have \(\mW_0^\rmA \mS_{1:i} = \mY_{1:i}\), i.e. the observations made up to this point. Similar reasoning applies for the inverse. Now, the posterior covariances do depend on \(\mW_0^\rmA\), respectively \(\mW_0^\rmH\) alone, but prior to convergence we only require \(\tr(\operatorname{Cov}[\rvx])\) for the stopping criterion. We show in \Cref{sec:stopping_criteria} under the assumptions of \Cref{thm:connection_cg} how to compute this at any iteration \(i\) independent of future actions and observations.

\paragraph{Uncertainty Calibration}

Generally the actions of \Cref{alg:problinsolve} identify eigenpairs \((\lambda_i, \vv_i)\) in descending order of \(\lambda_i \vv_i^\top \vr_0\) which is a well-known behavior of CG (see eqn. 5.29 in \citep{Nocedal2006}). In part, since this dynamic of the underlying Krylov subspace method is not encoded in the prior, the solver in its current form is typically miscalibrated (see also \cite{Cockayne2019}). While this non-linear information is challenging to include in the Gaussian framework, we can choose \(\mPhi\) and \(\mPsi\) in \eqref{eqn:prior_cov_class} to empirically calibrate uncertainty. This can be interpreted as a form of hyperparameter optimization similar to optimization of kernel parameters in GP regression.

We would like to encode prior knowledge about the way \(\mA\) and \(\mH\) act in the respective orthogonal spaces \(\operatorname{span}(\mS)^\perp\) and \(\operatorname{span}(\mY)^\perp\). For the Rayleigh quotient \(R(\mA, \vv) = (\vv^\top \mA \vv)(\vv^\top \vv)^{-1}\) it holds that \(\lambda_{\min}(\mA) \leq R(\mA, \vv) \leq \lambda_{\max}(\mA)\). Hence for vectors \(\vv\) lying in the respective null spaces of \(\mS\) and \(\mY\) our uncertainty should be determined by the not yet explored eigenvalues \(\lambda_{k+1}, \dots, \lambda_{n}\) of \(\mA\) and \(\mH\). Without prior information about the eigenspaces, we choose \(\mPhi = \phi \mI\) and \(\mPsi = \psi \mI\). If a priori we know the respective spectra, a straightforward choice is
\begin{equation*}
\phi = \psi^{-1} = \frac{1}{n-k}\sum_{i=k+1}^n \lambda_i(\mA).
\end{equation*}
In the absence of prior spectral information we can make use of already collected quantities during a run of \Cref{alg:problinsolve}. We build a one-dimensional regression model \(p(\ln R_i \mid \mY, \mS)\) for the \(\ln\)-Rayleigh quotient \(\ln R(\mA, \vs_i)\) given actions \(\vs_i\). Such a model can then encode the well studied behaviour of CG, whose Rayleigh coefficients rapidly decay at first, followed by a slower continuous decay \citep{Nocedal2006}. \Cref{fig:rayleigh_regression} illustrates this approach using a GP regression model. At convergence, we use the prediction of the Rayleigh quotient for the remaining \(n-k\) dimensions by choosing
\begin{equation*}
\phi = \psi^{-1}=\exp(\frac{1}{n-k} \sum_{i=k+1}^n \mathbb{E}[\ln R_i \mid \rmA, \mS]),
\end{equation*}
i.e. uncertainty about actions in \(\operatorname{span}(\mS)^\perp\) is calibrated to be the average Rayleigh quotient as an approximation to the spectrum. Depending on the application a simple or more complex model may be useful. For large problems, where generally \(k \ll n\), more sophisticated schemes become computationally feasible. However, these do not necessarily need to be computationally demanding due to the simple nature of this one-dimensional regression problem with few data. For example, approximate \cite{Karvonen2016} or even exact GP regression \cite{Solin2018} is possible in \(\mathcal{O}(k)\) using a Kalman filter.

\begin{figure}
  \begin{minipage}{\textwidth}
  \begin{minipage}{0.42\textwidth}
  \begin{figure}[H]
	\centering
	\includegraphics[width=\textwidth]{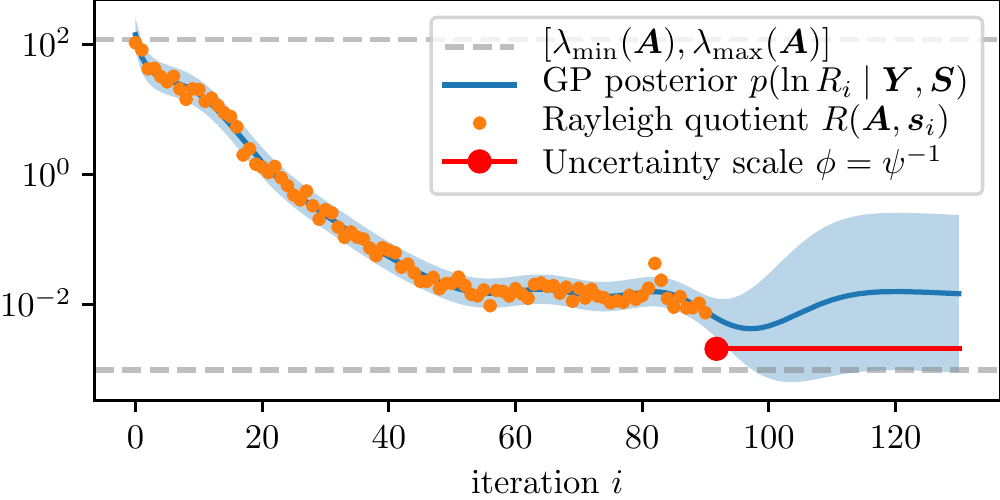}
	\caption{\textit{Rayleigh regression.} Uncertainty calibration via GP regression on \(\{\ln R(\mA, \vs_i)\}_{i=1}^k\) after \(k=91\) iterations of \Cref{alg:problinsolve} on an $n=1000$ dimensional M\'atern32 kernel matrix inversion problem. The degrees of freedom \(\phi = \psi^{-1} > 0\) are set based on the average predicted Rayleigh quotient for the remaining \(n-k = 909\) dimensions.
	\label{fig:rayleigh_regression}}
  \end{figure}
  \end{minipage}
  \hfill
  \begin{minipage}{0.56\textwidth}
  \vspace{-0.9em}
	\begin{table}[H]
	\caption{\textit{Uncertainty calibration for kernel matrices.} Monte Carlo estimate \(\bar{w}\approx \mathbb{E}_{\vx_*}[w(\vx_*)]\) measuring calibration given $10^5/n$ sampled linear problems of the form $(\mK + \varepsilon^2 \mI)\vx_* = \vb$ for each kernel and calibration method. For \(\bar{w} \approx 0\) the solver is well calibrated, for \(\bar{w} \gg 0\) underconfident and for \(\bar{w} \ll 0\) overconfident.\label{tab:unc_calib_test}}
		{\small
		\centering
		\begin{tabular}{lc*{4}{S[table-format=2.2, table-number-alignment=center]}}
			\toprule
			Kernel & $n$ & {none} & {Rayleigh} & {$\varepsilon^2$} & {$\overline{\lambda}_{k+1:n}$}\\
			\midrule
			Mat\' ern32 	& $10^2$		& -5.99	&  -0.24 &  0.32 &  0.09\\
			Mat\' ern32 	& $10^3$		& -1.93	&	7.53 &  4.26 &  4.19\\
			Mat\' ern32 	& $10^4$		&  3.87	&  17.16 &  8.48 &  8.47\\
			Mat\' ern52 	& $10^2$		& -7.84	&  -1.01 & -0.76 & -0.80\\
			Mat\' ern52 	& $10^3$		& -4.63	&	1.43 & -0.80 & -0.81\\
			Mat\' ern52 	& $10^4$		& -4.34	&  10.81 &  0.80 &  0.80\\
			RBF 	  		& $10^2$		& -7.53	&  -0.70 & -0.84 & -0.87\\
			RBF 			& $10^3$		& -4.94	&   6.60 &  0.77 &  0.77\\
			RBF 			& $10^4$		&  0.14	&  21.32 &  2.92 &  2.92\\
			\bottomrule
		\end{tabular}}
	\end{table}
  \end{minipage}
  \end{minipage}
\end{figure}

\section{Experiments}
\label{sec:experiments}
This section demonstrates the functionality of \Cref{alg:problinsolve}. We choose some -- deliberately simple -- example problems from machine learning and scientific computation, where the solver can be used to quantify uncertainty induced by finite computation, solve multiple consecutive linear systems, and propagate information between problems.

\paragraph{Gaussian Process Regression}
GP regression \citep{Rasmussen2006} infers a latent function \(f : \mathbb{R}^N \rightarrow \mathbb{R}\) from data \(\mD=(\mX, \vy)\), where \(\mX \in \mathbb{R}^{n \times N}\) and \(\vy \in \mathbb{R}^n\). Given a prior \(p(f) = \mathcal{GP}(f; 0, k)\) with kernel \(k\) for the unknown function \(f\), the posterior mean and marginal variance at \(m\) new inputs \(\tilde{\vx} \in \mathbb{R}^{N \times m}\) are \(\mathbb{E}[\tilde{\vf}] = \tilde{\vk}^\top (\mK + \varepsilon^2 \mI)^{-1} \vy\) and \(\mathbb{V}[\tilde{\vf}] = k(\tilde{\vx}, \tilde{\vx}) - \tilde{\vk}^\top (\mK + \varepsilon^2 \mI)^{-1} \tilde{\vk},\) where \(\mK = k(\mX, \mX) \in \mathbb{R}^{n \times n}\) is the Gram matrix of the kernel and \(\tilde{\vk} = k(\mX, \tilde{\vx}) \in \mathbb{R}^{n \times m}\). The bulk of computation during prediction arises from solving the linear system \((\mK + \varepsilon^2 \mI)\vz = \vb\) for some right-hand side \(\vb \in \mathbb{R}^n\) repeatedly. When using a probabilistic linear solver for this task, we can quantify the uncertainty arising from finite computation as well as the belief of the solver about the shape of the GP at a set of not yet computed inputs. \Cref{fig:GP_regression} illustrates this. In fact, we can estimate the marginal variance of the GP without solving the linear system again by multiplying \(\tilde{\vk}\) with the estimated inverse of \(\mK + \varepsilon^2 \mI\). In large-scale applications, we can trade off computational expense for increased uncertainty arising from the numerical approximation and quantified by the probabilistic linear solver. By assessing the numerical uncertainty arising from not exploring the full space, we can judge the quality of the estimated GP mean and marginal variance.

\begin{figure}
	\centering
	\includegraphics[width=\columnwidth]{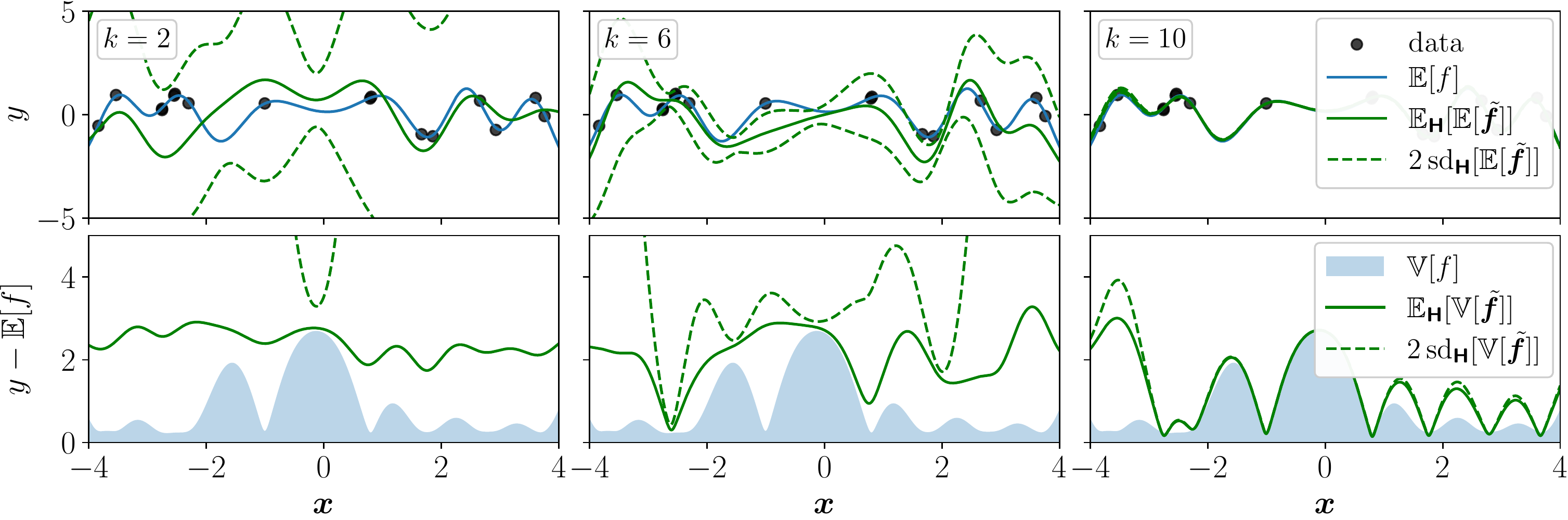}
	\caption{\textit{Numerical uncertainty in GP inference.} Computing posterior mean and covariance of a GP regression using a PLS. \emph{Top:} GP mean for a toy data set ($n=16$) computed with increasing number of iterations \(k\) of \Cref{alg:problinsolve}. The numerical estimate of the GP mean approaches the true mean. Note that the numerical variance is different from the marginal variance of the GP. \emph{Bottom:} GP variance and estimate of GP variance with numerical uncertainty. The GP variance estimate is computed using the estimated inverse from computing \(\mathbb{E}[\tilde{\vf}]\) \emph{without any additional solver iterations}.\label{fig:GP_regression}}
\end{figure}

\paragraph{Kernel Gram Matrix Inversion}
Consider a linear problem \(\mK \vx_* = \vb\), where \(\mK\) is generated by a Mercer kernel. For a \(\nu\)-times continuously differentiable kernel the eigenvalues \(\lambda_n(\mK)\) decay approximately as \(\abs{\lambda_n} \in \mathcal{O}(n^{-\nu - \frac{1}{2}})\) \cite{Weyl1912}. We can make use of this generative prior information by specifying a parametrized prior mean \(\mu(n) = \ln(\evtheta_0' n^{-\evtheta_1}) = \evtheta_0 - \evtheta_1 \ln(n)\) for the \(\ln\)-Rayleigh quotient model. Typically, such Gram matrices are ill-conditioned and therefore \(\mK' = \mK + \varepsilon^2 \mI\) is used instead, implying \(\lambda(\mK')_i \geq \varepsilon^2\). In order to assess calibration we apply various differentiable kernels to the airline delay dataset from January 2020 \cite{Airline2020}. We compute the \(\ln\)-ratio statistic \(w(\vx_*) = \frac{1}{2} \ln(\tr(\operatorname{Cov}[\rvx])) - \ln(\norm{\vx_* - \mathbb{E}[\rvx]}_2)\) for no calibration, calibration via Rayleigh quotient GP regression using \(\mu(n)\) as a prior mean, calibration by setting \(\phi=\varepsilon^2\) and calibration using the average spectrum \(\phi= \overline{\lambda}_{k+1:n}\). The average \(\bar{w}\) for \(10^5 / n\) randomly sampled test problems is shown in \Cref{tab:unc_calib_test}.\footnote{We decrease the number of samples with the dimension because forming \emph{dense} kernel matrices in memory and computing their eigenvalues becomes computationally prohibitive -- \emph{not} because of the cost of our solver.} Without any calibration the solver is generally overconfident. All tested calibration procedures reverse this, resulting in more cautious uncertainty estimates. We observe that Rayleigh quotient regression overcorrects for larger problems. This is due to the fact that its model correctly predicts \(\mK\) to be numerically singular from the dominant Rayleigh quotients, however it misses the information that the spectrum of \(\mK'\) is bounded from below by \(\varepsilon^2\). If we know the (average) of the remaining spectrum, significantly better calibration can be achieved, but often this information is not available. Nonetheless, since in this setting the majority of eigenvalues satisfy \(\lambda(\mK')_i \approx \varepsilon^2\) by choosing \(\phi=\psi^{-1} = \varepsilon^2\), we can get to the same degree of calibration. Therefore, we can improve the solver's uncertainty calibration at constant cost \(\mathcal{O}(1)\) per iteration. For more general problems involving Gram matrices without damping we may want to rely on Rayleigh regression instead.

\paragraph{Galerkin's Method for PDEs}
In the spirit of applying machine learning approaches to problems in the physical sciences and vice versa \citep{Carleo2019}, we use \Cref{alg:problinsolve} for the approximate solution of a PDE via Galerkin's method \cite{Fletcher1984}. Consider the Dirichlet problem for the Poisson equation given by
\begin{equation*}
\begin{cases}
-\Delta u(x,y) = f(x,y) &(x,y) \in \operatorname{int}\Omega\\
u(x,y) = u_{\partial \Omega}(x,y) &(x,y) \in \partial \Omega
\end{cases}
\end{equation*}
where \(\Omega\) is a connected open region with sufficiently regular boundary and \(u_{\partial \Omega} : \partial \Omega \rightarrow \mathbb{R}\) defines the boundary conditions. One obtains an approximate solution by projecting the weak formulation of the PDE to a finite dimensional subspace. This results in the \emph{Galerkin equation} \(\mA \vu=\vf\), i.e. a linear system where \(\mA\) is the Gram matrix of the associated bilinear form. \Cref{fig:PDE_discretization} shows the induced uncertainty on the solution of the Dirichlet problem for \(f(x,y) = 15\) and \(u_{\partial \Omega}(x,y) = (x^2 - 2y)^2 (1 + \sin(2 \pi x))\). The mesh and corresponding Gram matrix were computed using \textsc{FEniCS} \citep{Alnes2015}. We can exploit two properties of \Cref{alg:problinsolve} in this setting. First, if we need to solve multiple related problems \((\mA_j, \vf_j)_j\), by solving a single problem we obtain an estimate of the solution to all other problems. We can successively use the posterior over the inverse as a prior for the next problem. This approach is closely related to subspace recycling in numerical linear algebra \cite{Parks2006,Roos2017}. Second, suppose we first compute a solution in a low-dimensional subspace corresponding to a coarse discretization for computational efficiency. We can then leverage the estimated solution to extrapolate to an (adaptively) refined discretization based on the posterior uncertainty. In machine learning lingo these two approaches can be viewed as forms of \emph{transfer learning}.

\begin{figure*}
	\centering
    \begin{subfigure}[b]{0.19\textwidth}
        \centering
        \includegraphics[width=\textwidth]{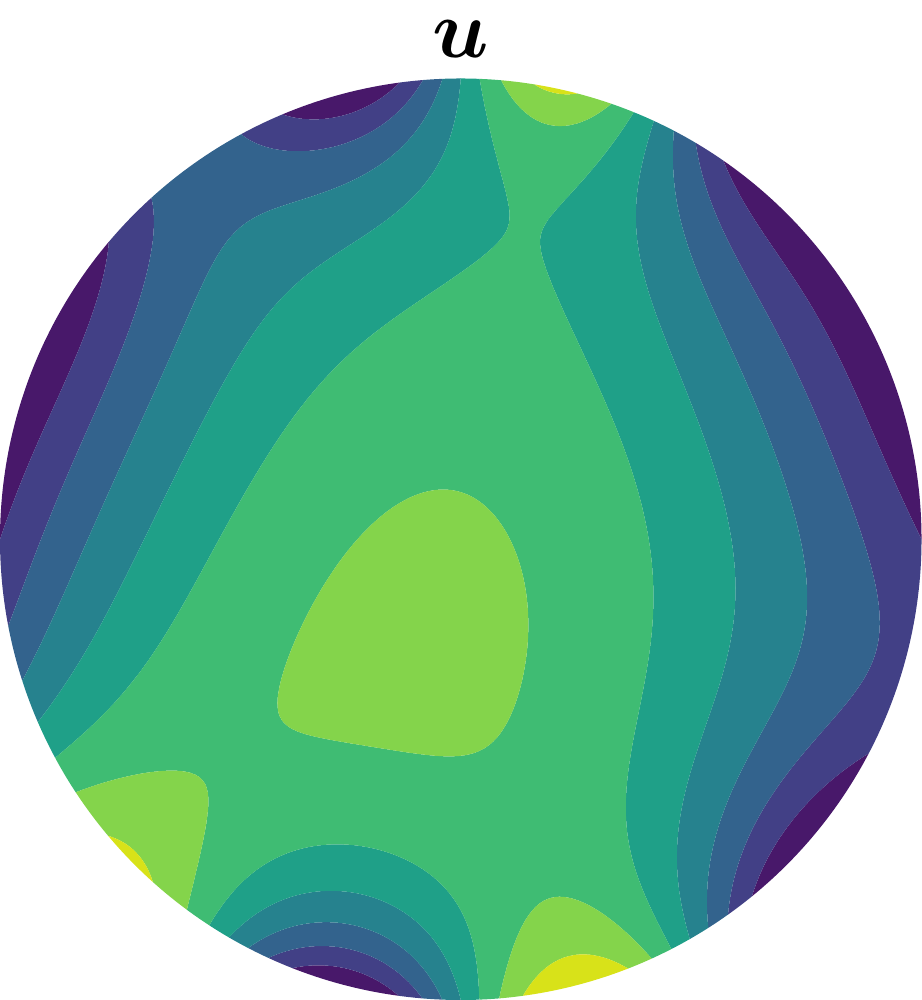}
        \caption{Ground truth\label{fig:PDE_fine_mesh}}
    \end{subfigure}%
    ~
    \begin{subfigure}[b]{0.285\textwidth}
        \centering
        \includegraphics[width=\textwidth]{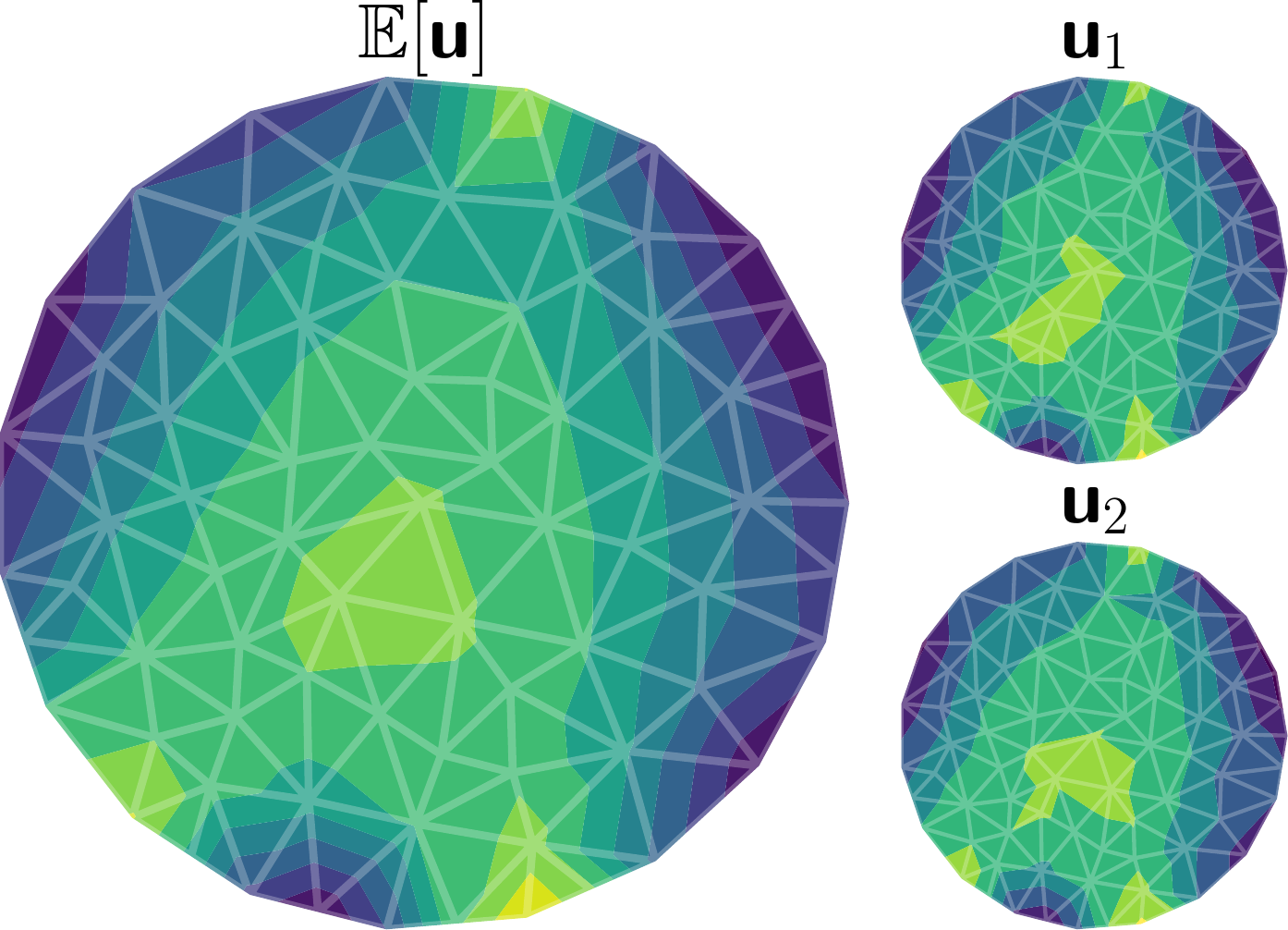}
        \caption{Solution mean \& samples\label{fig:PDE_mean_samples}}
    \end{subfigure}%
    ~
    \begin{subfigure}[b]{0.235\textwidth}
        \centering
        \includegraphics[width=\textwidth]{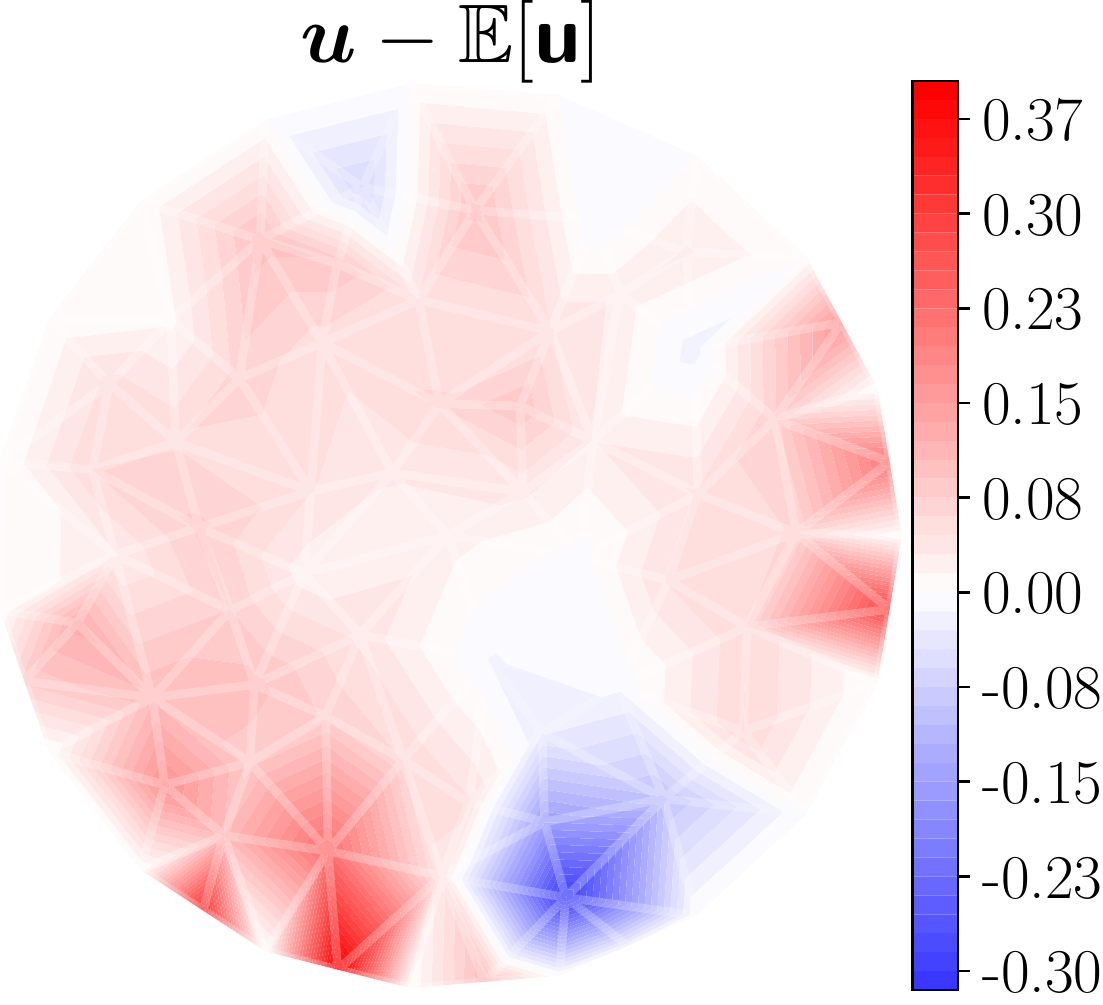}
        \caption{Signed error\label{fig:PDE_error}}
    \end{subfigure}
        ~
    \begin{subfigure}[b]{0.235\textwidth}
        \centering
        \includegraphics[width=\textwidth]{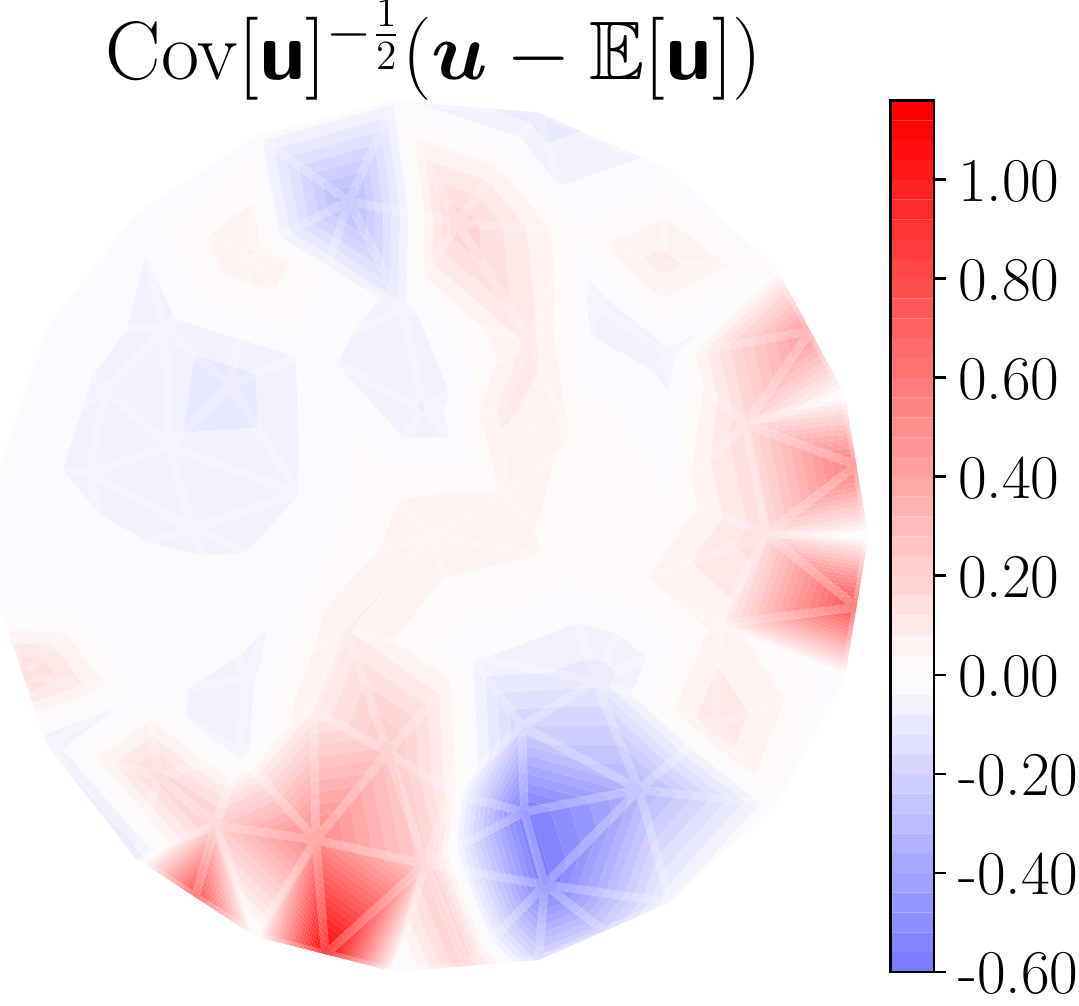}
        \caption{Uncertainty calibration\label{fig:PDE_uncertainty}}
    \end{subfigure}
	\caption{\textit{Solving the Dirichlet problem with a probabilistic linear solver.} \Cref{fig:PDE_fine_mesh,fig:PDE_mean_samples} show the ground truth and mean of the solution computed with \Cref{alg:problinsolve} after \(k=23\) iterations along with samples from the posterior. The posterior on the coarse mesh can be used to assess uncertainty about the solution on a finer mesh. The signed error computed on the coarse mesh in \Cref{fig:PDE_error} shows that the approximation is better near the top boundary of \(\Omega\). Given perfect uncertainty calibration, \Cref{fig:PDE_uncertainty} represents a sample from \(\mathcal{N}(\bm{0}, \mI)\). The apparent structure in the plot and smaller than expected deviations in the upper part of \(\Omega\) indicate the conservative confidence estimate of the solver.\label{fig:PDE_discretization}}
\end{figure*}

\section{Conclusion}
In this work, we condensed a line of previous research on probabilistic linear algebra into a self-contained algorithm for the solution of linear problems in machine learning. We proposed first principles to constrain the space of possible generative models and derived a suitable covariance class. In particular, our proposed framework incorporates prior knowledge on the system matrix or its inverse and performs inference for both in a \emph{consistent} fashion. Within our framework we identified parameter choices that recover the iterates of conjugate gradients in the mean, but add calibrated uncertainty around them in a computationally lightweight manner. To our knowledge our solver, available as part of the \href{https://github.com/probabilistic-numerics/probnum}{\textsc{ProbNum}} package, is the first practical implementation of this kind. In the final parts of this paper we showcased applications like kernel matrix inversion, where prior spectral information can be used for uncertainty calibration and outlined example use-cases for propagation of numerical uncertainty through computations. Naturally, there are also limitations remaining. While our theoretical framework can incorporate noisy matrix-vector product evaluations into its inference procedure via a Gaussian likelihood, practically \emph{tractable} inference in the inverse model is more challenging. Our solver also opens up new research directions. In particular, our outlined regression model on the Rayleigh quotient may lead to a probabilistic model of the eigenspectrum. Finally, the matrix-based view of probabilistic linear solvers could inform probabilistic approaches to matrix decompositions, analogous to the way Lanczos methods are used in the classical setting.

\section*{Broader Impact}

Our research on probabilistic linear solvers is primarily aimed at members of the machine learning field working on uncertainty estimation which use linear solvers as part of their toolkit. We are convinced that numerical uncertainty induced by finite computational resources is a key missing component to be quantified in machine learning settings. By making numerical uncertainty explicit like our solver does, holistic probabilistic models incorporating all sources of uncertainty become possible. In fact, we hope that this line of work stimulates further research into numerical linear algebra for machine learning, a topic that has been largely considered solved by the community.

This is first and foremost a methods paper aiming to improve the quantification of numerical uncertainty in linear problems. While methodological papers may seem far removed from application and questions of ethical and societal impact, this is not the case. Precisely due to the general nature of the problem setting, the linear solver presented in this work is applicable to a broad range of applications, from regression on flight data, to optimization in robotics, to the solution of PDEs in meteorology. The flip-side of this potential impact is that arguably, down the line, methodological research suffers from dual use more than any specialized field. While we cannot control the use of a probabilistic linear solver due to its general applicability, we have tried, to the best of our ability, to ensure it performs as intended.

We are hopeful that no specific population group is put at a disadvantage through this research. We are providing an open-source implementation of our method and of all experiments contained in this work. Therefore anybody with access to the internet is able to retrieve and reproduce our findings. In this manner we hope to adress the important issues of accessibility and reproducibility.

\begin{ack}
The authors gratefully acknowledge financial support by the European Research Council through ERC StG Action 757275 / PANAMA; the DFG Cluster of Excellence ``Machine Learning - New Perspectives for Science'', EXC 2064/1, project number 390727645; the German Federal Ministry of Education and Research (BMBF) through the T\" ubingen AI Center (FKZ: 01IS18039A); and funds from the Ministry of Science, Research and Arts of the State of Baden-W\" urttemberg.

JW is grateful to the International Max Planck Research School for Intelligent Systems (IMPRS-IS) for support.

We thank the reviewers for helpful comments and suggestions. JW would also like to thank Alexandra Gessner and Felix Dangel for a careful reading of an earlier version of this manuscript.
\end{ack}

\bibliographystyle{unsrtnat}
\bibliography{references}

\clearpage

\renewcommand{\thesection}{S\arabic{section}}
\renewcommand{\thesubsection}{\thesection.\arabic{subsection}}
\renewcommand{\theHsection}{S\arabic{section}} 

\newcommand{\beginsupplementary}{%
	\setcounter{table}{0}
	\renewcommand{\thetable}{S\arabic{table}}
	\setcounter{figure}{0}
	\renewcommand{\thefigure}{S\arabic{figure}}
	\setcounter{section}{0}
	\renewcommand{\theequation}{S\arabic{equation}} 
	\renewcommand{\thetheorem}{S\arabic{theorem}} 
	\renewcommand{\thedefinition}{S\arabic{definition}} 
	\renewcommand{\thecorollary}{S\arabic{corollary}} 
	\renewcommand{\theproposition}{S\arabic{proposition}} 
	\renewcommand{\theremark}{S\arabic{remark}} 
	\renewcommand{\thelemma}{S\arabic{lemma}} 
	\renewcommand{\theexample}{S\arabic{example}} 
}
\beginsupplementary

This supplement complements the paper \emph{Probabilistic Linear Solvers for Machine Learning} and is structured as follows. \Cref{sec:probabilistic_modelling} explains the approach of probabilistic numerics to model (deterministic) numerical problems probabilistically in more depth. \Cref{sec:kronecker_products} introduces different variants of Kronecker products used to define matrix-variate normal distributions in \Cref{sec:matrixvariate_normal}. \Cref{sec:problinsolvers} details the matrix-based inference procedure of probabilistic linear solvers based on matrix-vector product observations. It also contains some more explanation regarding prior construction and stopping criteria. \Cref{sec:theoretical_properties_proofs} and  \Cref{sec:proofs_prior_covariance_class} outline theoretical results from the paper and properties of the proposed covariance class, in particular detailed proofs. Finally, \Cref{sec:pde_discretization} provides some background for the application of probabilistic linear solvers to the solution of discretized partial differential equations. To provide a clear exposition to the reader in some sections we restate results from the literature. References referring to sections, equations or theorem-type environments within this document are tagged with `S', while references to, or results from the main paper are stated as is.

\paragraph{Preliminaries and Notation} We consider the linear system \(\mA \vx_* = \vb\), where \(\mA \in \mathbb{R}_{\textup{sym}}^{n \times n}\) is symmetric positive definite. The random variables \(\rmA, \rmH\) and \(\rvx\) model the linear operator \(\mA\), its inverse \(\mH=\mA^{-1}\) and the solution \(\vx_*\). \Cref{alg:problinsolve} chooses actions \(\mS = [\vs_1, \dots, \vs_k]  \in \mathbb{R}^{n \times k}\) given by its policy \(\pi(\vs \mid \rmA, \rmH, \rvx, \mA, \vb)\) and computes observations \(\mY = [\vy_1, \dots, \vy_k]  \in \mathbb{R}^{n \times k}\) given by a linear projection \(\vy_i = \mA \vs_i\) in each iteration \(0 < i \leq k\).

\section{Probabilistic Modelling of Deterministic Problems}
\label{sec:probabilistic_modelling}

At first glance it might seem counterintuitive to frame a numerical problem in the language of probability theory. After all, when considering the exact problem \(\mA \vx_* = \vb\) all quantities involved \(\mA, \vx_*,\) and \(\vb\) are deterministic. However, the distribution of the random variables \(\rmA, \rmH\) and \(\rvx\) represents \emph{epistemic uncertainty} arising from finite computational resources. With a finite budget only a limited amount of information can be obtained about \(\mA\) (e.g. via matrix-vector products). In particular, for a sufficiently large problem a priori the inverse \(\mH=\mA^{-1}\) and the solution \(\vx_*\), while deterministic and computable in finite time, are not known. This uncertainty about the inverse is captured by the prior distribution of \(\rmH\). In the Bayesian framework the belief about the inverse \(\rmH\) is then iteratively updated  given new observations \(\vy_i = \mA \vs_i\).

The motivation for also estimating \(\mA\) becomes clear if one considers the following. Usually in large-scale applications, the matrix \(\mA\) is never actually formed in memory due to computational constraints. Instead only the matrix-vector product \(\vv \mapsto \mA \vv\) is available. Therefore without further computation, the value of any given matrix entry \(\mA_{ij}\) is in fact uncertain. Further, generally other properties of the matrix \(\mA\) such as its eigenspectrum are also not readily available. The probabilistic framework provides a principled way of incorporating prior knowledge about \(\mA\) and makes assumptions about the problem explicit. Relating the prior model \(\rmA\) and \(\rmH\) is important here to allow \Cref{alg:problinsolve} to take such prior information into account in its policy. Finally, the strongest argument for a model \(\rmA\) may yet be the incorporation of noise. Suppose we only have access to \(\vy_i= (\mA +\mE_i) \vs_i\) with additive noise \(\mE_i \). This is a common occurrence in application, where the linear system to be solved arises from an approximation itself or if \(\mA\) is constructed from data. Concrete examples are batched empirical risk minimization problems or stochastic quadratic optimization. In this setting the probabilistic linear solver must estimate the true \(\mA\) via its observations.

The application of probabilistic inference to numerical problems goes back well into the last century \cite{Larkin1969, Diaconis1988, OHagan1992} and has recently seen a resurgence in research interest in the form of \emph{probabilistic numerics}. Overviews discussing motivations and historical perspectives can be found in \citet{Hennig2015a} and \citet{Oates2019}. \citet{Hennig2015} gives additional insight into the statistical interpretation of linear systems.

\section{The Kronecker Product and its Variants}
\label{sec:kronecker_products}

We will now introduce different types of Kronecker products needed for constructing covariances for matrix-variate distributions. In order to transfer results from probabilistic modelling of vector-variate random variables to the matrix-variate case, we need two types of vectorization operations, i.e. bijections between spaces of matrices and vector spaces.

Let \(\vec : \mathbb{R}^{m \times n} \rightarrow \mathbb{R}^{mn}\), denote the \emph{column-wise stacking operator} \cite{Henderson1981}, defined as
\begin{equation*}
\vec(\mX) = (\emX_{11}, \emX_{21}, \dots, \emX_{m1}, \emX_{12}, \dots, \emX_{mn})^\top \in \R^{mn}.
\end{equation*}
Further, define \(\operatorname{svec} : \mathbb{R}_{\textup{sym}}^{n \times n} \rightarrow \mathbb{R}^{\frac{1}{2}n(n+1)}\), the \emph{column-wise symmetric stacking operator} \cite{Alizadeh1998} given by
\begin{equation*}
\operatorname{svec}(\mX) = (\emX_{11}, \sqrt{2}X_{21}, \dots, \sqrt{2}\emX_{n1}, X_{22}, \sqrt{2}\emX_{32}, \dots, \sqrt{2}X_{n2}, \dots, X_{nn})^\top \in \R^{\frac{1}{2}n (n+1)}.
\end{equation*}
To translate between the two representations following \citet{Schaecke2013} we also define the matrix \(\mQ \in \mathbb{R}^{\frac{1}{2}n(n+1) \times n^2}\) such that for all symmetric matrices \(\mX \in \mathbb{R}_{\textup{sym}}^{n \times n}\), we have \(\mQ \vec(\mX) = \operatorname{svec}(\mX)\) and \(\vec(\mX) = \mQ^\top \operatorname{svec}(\mX)\). Note, that \(\mQ\) has orthonormal rows, i.e. \(\mQ \mQ^\top = \mI\). For convenience we also name the inverse operations \(\mat \coloneqq \vec^{-1}\) and \(\smat \coloneqq \svec^{-1}\).

\subsection{Kronecker Product}
We make extensive use of Kronecker-type structures for covariance matrices of matrix-variate distributions in this paper. The \emph{Kronecker product} \(\mA \otimes \mB\) \cite{Loan2000} of two matrices \(\mA \in \R^{m_1 \times n_1}\) and \(\mB \in \R^{m_2 \times n_2}\) is given by
\begin{equation*}
\mA \otimes \mB =
\begin{pmatrix}
\mA_{11} \mB & \dots & \mA_{1n_1} \mB\\
\vdots & \ddots & \vdots\\
\mA_{m_11} \mB & \dots & \mA_{m_1n_1} \mB\\
\end{pmatrix} \in \R^{(m_1m_2) \times (n_1 n_2)}
\end{equation*}
The Kronecker product satisfies the characteristic property
\begin{equation}
\label{eqn:kronecker_charprop}
(\mA \otimes \mB) \vec(\mX) = \vec(\mB \mX \mA^\top),
\end{equation}
for \(\mX \in \R^{n_2 \times n_1}\). Characteristic properties of Kronecker-type products are useful to turn matrix equations into vector equations. We state a set of properties of the Kronecker product next without proof. More detail on Kronecker products can be found in \citet{Loan2000}.

\begin{proposition}[Properties of the Kronecker Product \cite{Loan2000}]
\label{prop:kron_properties}
The Kronecker product satisfies the following identities:
\begin{align}
\exists \mA, \mB : \mA \otimes \mB &\neq \mB \otimes \mA\\
(\mA \otimes \mB)^\top  &= \mA^\top \otimes \mB^\top\\
(\mA \otimes \mB)^{-1}  &= \mA^{-1} \otimes \mB^{-1}\\
(\mA + \mB) \otimes \mC &= \mA \otimes \mC + \mB \otimes \mC\\
(\mA \otimes \mB)(\mC \otimes \mD) &= (\mA \mC) \otimes (\mB \mD)\label{eqn:kron_products}\\
\trace (\mA \otimes \mB) &= \trace (\mA) \trace (\mB)\\
\mA \in \mathbb{R}_{\textup{sym}}^{m \times m}, \mB \in \mathbb{R}_{\textup{sym}}^{n \times n} &\implies \mA \otimes \mB \in \mathbb{R}_{\textup{sym}}^{mn \times mn}\\
\mA \otimes \mB = (\mL_\mA \mL_\mA^\top) \otimes (\mL_\mB \mL_\mB^\top) &= (\mL_\mA \otimes \mL_\mB)(\mL_\mA^\top \otimes \mL_\mB^\top)\\
\mA \otimes \mB = (\mU_\mA \mLambda_\mA \mU_\mA^\top) \otimes (\mU_\mB \mLambda_\mB \mU_\mB^\top) &= (\mU_\mA \otimes \mU_\mB)(\mLambda_\mA \otimes \mLambda_\mB)(\mU_\mA^\top \otimes \mU_\mB^\top)
\end{align}
\end{proposition}

\subsection{Box Product}
The \emph{box product} \(\mA \boxtimes \mB \in \R^{(m_1m_2) \times (n_1n_2)}\) can be defined via its characteristic property
\begin{equation}
\label{eqn:box_charprop}
(\mA \boxtimes \mB) \vec(\mY) = \vec(\mB \mY^\top \mA^\top)
\end{equation}
for \(\mY \in \R^{n_1 \times n_2}\). See also \citet{Olsen2012} for details.

\begin{proposition}[Properties of the Box Product \cite{Olsen2012}]
\label{prop:box_properties}
The box product satisfies the following identities:
\begin{align}
\exists \mA, \mB : \mA \boxtimes \mB &\neq \mB \boxtimes \mA\\
(\mA \boxtimes \mB)^\top  &= \mB^\top \boxtimes \mA^\top\\
(\mA \boxtimes \mB)^{-1}  &= \mB^{-1} \boxtimes \mA^{-1}\\
(\mA + \mB) \boxtimes \mC &= \mA \boxtimes \mC + \mB \boxtimes \mC\\
(\mA \boxtimes \mB)(\mC \boxtimes \mD) &= (\mA \mD) \otimes (\mB \mC)\\
(\mA \boxtimes \mB)(\mC \otimes \mD) &= (\mA \mD) \boxtimes (\mB \mC)\\
(\mA \otimes \mB)(\mC \boxtimes \mD) &= (\mA \mC) \boxtimes (\mB \mD)\\
\trace (\mA \boxtimes \mB) &= \trace (\mA \mB)
\end{align}
\end{proposition}

\subsection{Symmetric Kronecker Product}

The \emph{symmetric Kronecker product} \(\mA \ostimes \mB\) of two square matrices \(\mA, \mB \in \R^{n \times n}\) is defined via its characteristic property for \(\mX \in \mathbb{R}_{\textup{sym}}^{n \times n}\) as
\begin{equation}
\label{eqn:symm_kronecker_charprop}
(\mA \ostimes \mB) \operatorname{svec}(\mX) = \frac{1}{2}\operatorname{svec}(\mB \mX \mA^\top + \mA \mX \mB^\top)
\end{equation}
or equivalently
\begin{equation*}
\mA \ostimes \mB= \frac{1}{2}\mQ(\mA \otimes \mB + \mB \otimes \mA) \mQ^\top.
\end{equation*}

\begin{proposition}[Properties of the Symmetric Kronecker Product \cite{Alizadeh1998,Schaecke2013}]
\label{prop:symkron_properties}
The symmetric Kronecker product satisfies the following identities:

\begin{align}
\mA \ostimes \mB &= \mB \ostimes \mA\\
(\mA \ostimes \mB)^\top  &= \mA^\top \ostimes \mB^\top\\
(\mA \ostimes \mA)^{-1}  &= \mA^{-1} \ostimes \mA^{-1}\\
(\mA + \mB) \ostimes \mC &= \mA \ostimes \mC + \mB \ostimes \mC\\
(\mA \ostimes \mB)(\mC \ostimes \mD) &= \frac{1}{2}(\mA \mC \ostimes \mB \mD + \mA \mD \ostimes \mB \mC)\\
\mA \in \mathbb{R}_{\textup{sym}}^{n \times n}, \mB \in \mathbb{R}_{\textup{sym}}^{n \times n} &\implies \mA \ostimes \mB \in \mathbb{R}_{\textup{sym}}^{\frac{1}{2}n (n+1) \times \frac{1}{2}n (n+1)}\\
\mA \ostimes \mA = (\mL_\mA \mL_\mA^\top) \ostimes (\mL_\mA \mL_\mA^\top) &= (\mL_\mA \ostimes \mL_\mA)(\mL_\mA^\top \ostimes \mL_\mA^\top)\\
\mA \ostimes \mA = (\mU_\mA \mLambda_\mA \mU_\mA^\top) \ostimes (\mU_\mA \mLambda_\mA \mU_\mA^\top) &= (\mU_\mA \ostimes \mU_\mA)(\mLambda_\mA \ostimes \mLambda_\mA)(\mU_\mA^\top \ostimes \mU_\mA^\top)
\end{align}
Note, that the symmetric Kronecker product represented as a \(\frac{1}{2}n (n+1) \times \frac{1}{2}n (n+1)\) matrix is in general not symmetric.
\end{proposition}
Further properties can be found in \citet{Alizadeh1998} and \citet{Schaecke2013}. We prove the following technical results for mixed expressions of Kronecker-type products, which we will make use of later.

\begin{corollary}[Mixed Kronecker Product Identities]
\label{cor:mixed_kronecker_identities}
Let \(\mA \in \mathbb{R}_{\textup{sym}}^{n \times n}\), \(\mB, \mC \in \mathbb{R}^{n \times k}\) and \(\mX \in \R^{k \times k}\) such that \((\mC \mX \mB^\top)^\top = \mC \mX \mB^\top\), then it holds that
\begin{align}
\mQ^\top(\mA \ostimes \mA)\mQ(\mB \otimes \mC) \vec(\mX) &= \frac{1}{2}(\mA\mB \otimes \mA\mC + \mA\mC \boxtimes \mA\mB)\vec(\mX)\label{eqn:symkron_kron}\\
(\mB^\top \otimes \mC^\top)\mQ^\top(\mA \ostimes \mA)\mQ &= \frac{1}{2}(\mB^\top \mA \otimes \mC^\top \mA + \mB^\top\mA \boxtimes \mC^\top \mA).\\
(\mB ^\top \otimes \mC ^\top)\mQ^\top(\mA \ostimes \mA)\mQ(\mB  \otimes \mC )\vec(\mX) &= \frac{1}{2}(\mB ^\top \mA \mB  \otimes \mC ^\top \mA \mC  + \mB ^\top \mA \mC  \boxtimes \mC ^\top \mA \mB )\vec(\mX).\label{eqn:kron_symkron_kron}
\end{align}
Now, assume \(\mA\) to be invertible, \(\rank (\mC) = k\) and \(\mY \in \R^{k \times n}\) such that \((\mY \mC)^\top = \mY\mC\), then for
\begin{align*}
\mG &=(\mI_n \otimes \mC^\top)\mQ^\top(\mA \ostimes \mA)\mQ(\mI_n  \otimes \mC )\\
\mG_{\textup{right}}^{-1} &= (2\mA^{-1} - \mC(\mC^\top \mA \mC)^{-1}\mC^\top) \otimes (\mC^\top \mA \mC)^{-1}
\end{align*}
we have \(\mG \mG_{\textup{right}}^{-1}\vec(\mY) = \vec(\mY)\), i.e. \(\mG_{\textup{right}}^{-1}\) is the right inverse of \(\mG\). Finally, for  \(\mD, \mE \in \mathbb{R}^{n \times n}\) and \(\mZ \in \mathbb{R}_{\textup{sym}}^{n \times n}\) such that \((\mE \mA \mZ \mA \mD^\top)^\top = \mE \mA \mZ \mA \mD^\top\), we have
\begin{equation}
\label{eqn:symkron_kron_symkron}
(\mA^\top \ostimes \mA^\top)\mQ(\mD \otimes \mE)\mQ^\top(\mA \ostimes \mA) \svec(\mZ)= (\mA^\top \mD \mA) \ostimes (\mA^\top \mE \mA) \svec(\mZ).
\end{equation}
\end{corollary}
\begin{proof}
Let \(\mX \in \mathbb{R}^{k \times k}\) such that \((\mC \mX \mB^\top)^\top = \mC \mX \mB^\top\), then
\begin{align*}
\mQ^\top(\mA \ostimes \mA)\mQ(\mB \otimes \mC) \vec(\mX) &= \mQ^\top(\mA \ostimes \mA)\mQ \vec(\mC \mX \mB^\top)\\
	&= \mQ^\top(\mA \ostimes \mA)\svec(\mC \mX \mB^\top)\\
	&= \mQ^\top \svec(\mA \mC \mX \mB^\top \mA)\\
	&= \frac{1}{2} \vec(\mA\mC \mX \mB^\top \mA + \mA \mB \mX^\top \mC^\top \mA)\\
	&= \frac{1}{2}(\mA \mB \otimes \mA \mC + \mA \mC \boxtimes \mA \mB),\\
\intertext{further it holds for \(\mW \in \mathbb{R}_{\textup{sym}}^{n \times n}\)}
(\mB^\top \otimes \mC^\top)\mQ^\top(\mA \ostimes \mA)\mQ \vec(\mW)&= (\mB^\top \otimes \mC^\top)\mQ^\top \svec(\mA \mW \mA)\\
	&= \vec(\mC^\top \mA \mW \mA \mB)\\
	&= \frac{1}{2}(\mC^\top \mA \mW \mA \mB + \mC^\top \mA^\top \mW^\top \mA^\top \mB)\\
	&= \frac{1}{2}(\mB^\top \mA \otimes \mC^\top \mA + \mB^\top\mA \boxtimes \mC^\top \mA),\\
\intertext{and using the properties of the Kronecker and the Box product we obtain}
(\mB ^\top \otimes \mC ^\top)\mQ^\top(\mA \ostimes \mA)\mQ(\mB  \otimes \mC )\vec(\mX) &= (\mB^\top \otimes \mC^\top) \frac{1}{2}(\mB^\top \mA \otimes \mC^\top \mA + \mB^\top\mA \boxtimes \mC^\top \mA)\vec(\mX)\\
	&= \frac{1}{2}(\mB^\top \mA \otimes \mC^\top \mA + \mB^\top\mA \boxtimes \mC^\top \mA)\vec(\mX).
\end{align*}
Now let \(\mA\) be invertible, let \(\mC\) have full rank and choose \(\mY \in \mathbb{R}^{k \times n}\) arbitrarily such that \((\mY \mC)^\top = \mY\mC\). Then using \Cref{prop:kron_properties} and \Cref{prop:box_properties} we obtain
\begin{align*}
(\mI_n &\otimes \mC^\top)\mQ^\top(\mA \ostimes \mA)\mQ(\mI_n  \otimes \mC )(2\mA^{-1} - \mC(\mC^\top \mA \mC)^{-1}\mC^\top) \otimes (\mC^\top \mA \mC)^{-1} \vec(\mY)\\
	&= \frac{1}{2}(\mA  \otimes \mC ^\top \mA \mC  + \mA \mC  \boxtimes \mC ^\top \mA )(2\mA^{-1} - \mC(\mC^\top \mA \mC)^{-1}\mC^\top) \otimes (\mC^\top \mA \mC)^{-1}) \vec(\mY)\\
	&= (\mI_n \otimes \mI_k - \frac{1}{2} \mA \mC(\mC^\top \mA \mC)^{-1} \mC^\top \otimes \mI_k + \mA \mC(\mC^\top \mA \mC)^{-1} \boxtimes \mC^\top - \frac{1}{2}\mA \mC(\mC^\top \mA \mC)^{-1}\boxtimes \mC^\top) \vec(\mY)\\
	&= (\mI_n \otimes \mI_k - \frac{1}{2} \mA \mC(\mC^\top \mA \mC)^{-1} \mC^\top \otimes \mI_k + \frac{1}{2}\mA \mC(\mC^\top \mA \mC)^{-1} \boxtimes \mC^\top) \vec(\mY)\\
	&= \vec(\mY) - \frac{1}{2}(\mY \mC (\mC^\top \mA \mC)^{-1} \mC^\top \mA - \mC^\top \mY^\top (\mC^\top \mA \mC)^{-1} \mC^\top \mA)\\
	&= \vec(\mY)
\end{align*}
Lastly, by assumption it holds that
\begin{align*}
(\mA^\top \ostimes \mA^\top)\mQ(\mD \otimes \mE)\mQ^\top(\mA \ostimes \mA) \svec(\mZ) &= (\mA \ostimes \mA) \mQ \vec(\mE \mA \mZ \mA \mD^\top)\\
	&= \svec(\mA \mE \mA \mZ \mA \mD^\top \mA)\\
	&= \frac{1}{2} (\mA \mE \mA \mZ \mA \mD^\top \mA + \mA \mD \mA \mZ \mA \mE^\top \mA)\\
	&= (\mA \mD \mA \ostimes \mA \mE \mA)\svec(\mZ).
\end{align*}
This concludes the proof.
\end{proof}

\section{The Matrix-variate Normal Distribution}
\label{sec:matrixvariate_normal}

In order for our probabilistic linear solvers to infer the true latent \(\mA\) or its inverse \(\mH=\mA^{-1}\), we need a distribution expressing the belief of the solver over those latent quantities at any given point. A Gaussian distribution over matrices will play this role, motivated by the linear nature of the observations. This section closely follows \citet{Gupta2000}.

\begin{definition}[Matrix-variate Normal Distribution \cite{Gupta2000}]
\label{def:matrixnormal}
Let \(\mX_0 \in \mathbb{R}^{m \times n}\) and let \(\mV \in \mathbb{R}_{\textup{sym}}^{m}\) and \(\mW \in \mathbb{R}_{\textup{sym}}^{n \times n}\) be positive-definite. We say a random matrix \(\rmX\) has a \emph{matrix-variate normal distribution} with mean \(\mX_0\) and covariance \(\mV \otimes \mW\), iff
\begin{equation*}
\operatorname{vec}(\rmX^\top) \sim \mathcal{N}_{mn}(\operatorname{vec}(\mX_0^\top), \mV \otimes \mW).
\end{equation*}
We write as a shorthand \(\rmX \sim \mathcal{N}(\mX_0, \mV \otimes \mW)\).
\end{definition}
Note, that the matrices \(\mV\) and \(\mW\) represent the covariance between rows and columns of \(\rmX\), respectively.  Since we model symmetric matrices in this work, we also introduce a Gaussian distribution over \(\mathbb{R}_{\textup{sym}}^{n \times n}\).

\begin{definition}[Symmetric Matrix-variate Normal Distribution \cite{Gupta2000}]
\label{def:symmatrixnormal}
Let \(\mX_0, \mW \in \mathbb{R}_{\textup{sym}}^{n \times n}\) such that \(\mW\) is positive-definite, then the random matrix \(\rmX\) has a \emph{symmetric matrix-variate normal distribution}, iff
\begin{equation*}
\operatorname{svec}(\rmX) \sim \mathcal{N}_{\frac{1}{2}n (n+1)}(\operatorname{svec}(\mX_0), \mW \ostimes \mW).
\end{equation*}
We write \(\rmX \sim \mathcal{N}(\mX_0, \mW \ostimes \mW)\).
\end{definition}
It follows immediately from the definition that realizations of a symmetric matrix-variate normal distribution are symmetric matrices. This distribution also emerges naturally by conditioning a matrix-variate normal distribution on the linear constraint \(\rmX = \rmX^\top\).

\section{Probabilistic Linear Solvers}
\label{sec:problinsolvers}

Probabilistic linear solvers (PLS) \cite{Hennig2015, Cockayne2019, Bartels2019} infer posterior beliefs over the matrix \(\mA\), its inverse \(\mH\) or the solution \(\vx_* = \mH \vb\) of a linear system via linear observations \(\mY=\mA \mS\). We consider matrix-based inference \cite{Bartels2019} in this work. Assuming a prior \(p(\rmA)\) or \(p(\rmH)\), actions \(\mS\) and linear observations \(\mY\) such methods return posterior distributions \(p(\rmA \mid \mS, \mY)\) or \(p(\rmH \mid \mS, \mY)\).

\subsection{Matrix-based Inference}
\label{sec:matrixbased_inference}

The generic matrix-based inference procedure of probabilistic linear solvers is a consequence of the matrix-variate version of the following standard result for Gaussian inference under linear observations.

\begin{theorem}[Linear Gaussian Inference \cite{Bishop2006}]
\label{thm:gaussian_inference}
Let \(\vv \sim \mathcal{N}(\vmu, \mSigma)\), where \(\vmu \in \mathbb{R}^n\) and \(\mSigma \in \mathbb{R}_{\textup{sym}}^{n \times n}\) positive-definite, and assume we are given observations of the form
\[\mB \vv + \vb = \vy \in \mathbb{R}^m,\] where \(\mB \in \R^{m \times n}\) and \(\vb \in \R^{m}\). Assuming a Gaussian likelihood
\begin{equation*}
p(\vy \mid \mB, \vv, \vb) = \mathcal{N}(\vy; \mB \vv + \vb, \mLambda),
\end{equation*}
for \( \mLambda \in \mathbb{R}_{\textup{sym}}^{m}\) positive definite, results in the posterior distribution
\begin{align*}
p(\vv \mid \vy, \mB, \vb) = \mathcal{N} \big(\vv; \, &\vmu + \mSigma \mB^\top(\mB \mSigma \mB^\top + \mLambda)^{-1} (\vy - \mB \vmu - \vb),\\
	&\mSigma - \mSigma \mB^\top (\mB \mSigma \mB^\top + \mLambda)^{-1}\mB \mSigma \big).
\end{align*}
Further, the marginal distribution of \(\vy\) is given by
\begin{equation*}
p(\vy) = \mathcal{N}(\vy; \mB \vmu + \vb,  \mB \mSigma \mB^\top + \mLambda).
\end{equation*}
\end{theorem}

\subsubsection{Asymmetric Model}

\begin{corollary}[Asymmetric matrix-based Gaussian Inference \cite{Hennig2013, Hennig2015, Bartels2019}]
\label{cor:asymmetric_gaussian_inference}
Assume a prior \(p(\rmA) = \mathcal{N}(\rmA; \mA_0, \mV_0 \otimes \mW_0)\) and exact observations of the form \(\mY= \mA \mS\), corresponding to a Dirac likelihood \(p(\mY \mid \rmA, \mS) = \delta(\mY - \mA \mS)\), then the posterior \(p(\rmA \mid \mS, \mY)  = \mathcal{N}(\rmA; \mA_k, \mSigma_k)\) is given by
\begin{align*}
\mA_k &= \mA_0 + \mDelta_0 \mU^{\top}\\
\mSigma_k &= \mV_0 \otimes \mW_0(\mI_n - \mS \mU^\top)
\end{align*}
where \(\mDelta_0 = \mY - \mA_0\mS\) and \(\mU = \mW_0 \mS(\mS^\top \mW_0 \mS)^{-1}\).
\end{corollary}
\begin{proof}
In vectorized form the likelihood is given by
\begin{equation*}
p(\vec(\mY^\top) \mid \vec(\rmA^\top), \vec(\mS^\top)) = \delta(\vec(\mY^\top) - \vec(\mS^\top \mA^\top)) = \delta(\vec(\mY^\top) - (\mI \otimes \mS^\top) \vec(\mA^\top))
\end{equation*}
Using the \Cref{def:matrixnormal} of the matrix-variate normal distribution, applying \Cref{thm:gaussian_inference} and using property \eqref{eqn:kron_products} of the Kronecker product in \Cref{prop:kron_properties} leads to
\begin{align*}
\vec (\mA_k^\top) &= \vec(\mA_0^\top) + (\mV_0 \otimes \mW_0)(\mI \otimes \mS)((\mI \otimes \mS^\top)(\mV_0 \otimes \mW_0)(\mI \otimes \mS))^{-1} (\vec(\mY^\top) -(\mI \otimes \mS^\top)\vec(\mA_0^\top))\\
&= \vec(\mA_0^\top) + (\mV_0 \otimes \mW_0 \mS)(\mV_0 \otimes \mS^\top \mW_0 \mS )^{-1} \vec(\mDelta_0^\top)\\
&= \vec(\mA_0^\top) + (\mI_n \otimes \mW_0 \mS(\mS^\top \mW_0 \mS )^{-1}) \vec(\mDelta_0^\top)\\
&= \vec(\mA_0^\top +\mU \mDelta_0^\top )
\end{align*}
and further analogously, additionally using bilinearity of the Kronecker product, we obtain
\begin{align*}
\mSigma_k &= \mV_0 \otimes \mW_0 - (\mV_0 \otimes \mW_0)(\mI \otimes \mS)((\mI \otimes \mS^\top)(\mV_0 \otimes \mW_0)(\mI \otimes \mS))^{-1}(\mI \otimes \mS^\top)(\mV_0 \otimes \mW_0) \\
&= \mV_0 \otimes \mW_0 - (\mV_0 \otimes \mW_0 \mS) (\mV_0 \otimes \mS^\top \mW_0 \mS)^{-1} (\mV_0 \otimes \mS^\top \mW_0)\\
&= \mV_0 \otimes \mW_0 - \mV_0 \otimes (\mW_0 \mS( \mS^\top \mW_0 \mS)^{-1} \mS^\top \mW_0)\\
&= \mV_0 \otimes \mW_0 ( \mI - \mS \mU^\top).
\end{align*}
This concludes the proof.
\end{proof}

\subsubsection{Symmetric Model}

\begin{corollary}[Symmetric Matrix-based Gaussian Inference \cite{Hennig2013, Hennig2015, Bartels2019}]
\label{thm:sym_gaussian_inference}
Assume a symmetric prior \(p(\rmA) = \mathcal{N}(\rmA; \mA_0, \mW_0 \ostimes \mW_0)\) and exact observations of the form \(\mY= \mA \mS\), corresponding to a Dirac likelihood \(p(\mY \mid \rmA, \mS) = \delta(\mY - \mA \mS)\), then the posterior \(p(\rmA \mid \mS, \mY)  = \mathcal{N}(\rmA; \mA_k, \mSigma_k)\) is given by
\begin{align*}
\mA_k &= \mA_0 + \mDelta_0 \mU^\top + \mU \mDelta_0^\top - \mU \mS^\top \mDelta_0 \mU^\top
= \mA_0 + \mU \mV^\top + \mV \mU^\top\\
\mSigma_k &=\mW_0 (\mI_n-\mS\mU^\top) \ostimes \mW_0 (\mI_n-\mS\mU^\top)
\end{align*}
where \(\mDelta_0 = \mY - \mA_0\mS\), \(\mU = \mW_0 \mS (\mS^\top \mW_0 \mS)^{-1}\) and \(\mV = (\mI_n - \frac{1}{2}\mU \mS^\top)\mDelta_0\).
\end{corollary}
\begin{proof}
A proof can be found in the appendix of \citet{Hennig2015}. We rederive it here in our notation. By assumption the likelihood takes the vectorized form
\begin{equation*}
p(\vec(\mY^\top) \mid \svec(\rmA), \vec(\mS^\top)) = \delta(\vec(\mY^\top)-\vec(\mS^\top \mA^\top)) = \delta(\vec(\mY^\top) - (\mI \otimes \mS^\top)\mQ^\top \svec(\mA))
\end{equation*}
Applying \Cref{thm:gaussian_inference} gives
\begin{align*}
\svec(\mA_k) &= \svec(\mA_0) + (\mW_0 \ostimes \mW_0) \mQ(\mI_n \otimes \mS)\mG^{-1}(\vec(\mY^\top) - (\mI \otimes \mS^\top)\mQ^\top \svec(\mA_0))\\
&= \svec(\mA_0) + (\mW_0 \ostimes \mW_0) \mQ(\mI_n \otimes \mS)\mG^{-1}\vec(\mDelta_0^\top)\\
\mSigma_k &= \mW_0 \ostimes \mW_0 - (\mW_0 \ostimes \mW_0) \mQ(\mI_n \otimes \mS)\mG^{-1}(\mI_n \otimes \mS^\top)\mQ^\top (\mW_0 \ostimes \mW_0),
\end{align*}
where \(\mDelta_0 = \mY - \mA_0\mS\) and the Gram matrix is given by
\begin{equation*}
\mG = (\mI_n \otimes \mS^\top)\mQ^\top (\mW_0 \ostimes \mW_0)\mQ(\mI_n \otimes \mS) \in \R^{nk \times nk}.
\end{equation*}
Now since \((\mDelta_0^\top \mS)^\top = \mDelta_0^\top \mS\), we have by \Cref{cor:mixed_kronecker_identities} that the right inverse of \(\mG\) is given by
\begin{equation*}
\mG_{\textup{right}}^{-1} = (2\mW_0^{-1} - \mS(\mS^\top \mW_0 \mS)^{-1}\mS^\top) \otimes (\mS^\top \mW_0 \mS)^{-1}
\end{equation*}
and therefore using \eqref{eqn:kron_products} and \eqref{eqn:symkron_kron} we obtain
\begin{align*}
\svec(\mA_k) &= \svec(\mA_0) + (\mW_0 \ostimes \mW_0) \mQ(\mI_n \otimes \mS)\mG_{\textup{right}}^{-1} \vec(\mDelta_0^\top)\\
&= \svec(\mA_0) + \mQ\mQ^\top(\mW_0 \ostimes \mW_0) \mQ (2\mW_0^{-1} - \mS(\mS^\top \mW_0 \mS)^{-1}\mS^\top) \otimes \mS(\mS^\top \mW_0 \mS)^{-1} \vec(\mDelta_0^\top)\\
&= \svec(\mA_0) + \mQ \frac{1}{2}\big((2\mI - \mU \mS^\top) \otimes \mU + \mU \boxtimes (2\mI - \mU \mS^\top)\big) \vec(\mDelta_0^\top)\\
&= \svec(\mA_0) + \svec(\mU \mDelta_0^\top (\mI - \frac{1}{2}\mU\mS^\top)^\top + (\mI - \frac{1}{2}\mU\mS^\top)\mDelta_0 \mU^\top)\\
&=\svec(\mA_0 + \mDelta_0 \mU^\top + \mU \mDelta_0^\top - \mU \mS^\top \mDelta_0 \mU^\top).
\end{align*}
Further by definition it holds that
\begin{equation*}
\mU \mV^\top + \mV \mU^\top = \mU \mDelta_0^\top(\mI_n - \frac{1}{2}\mS \mU^\top) + (\mI_n - \frac{1}{2}\mU \mS^\top)\mDelta_0 \mU^\top = \mDelta_0 \mU^\top + \mU \mDelta_0^\top - \mU \mS^\top \mDelta_0 \mU^\top.
\end{equation*}
For the covariance we obtain using the right inverse of the Gram matrix and \eqref{eqn:symkron_kron_symkron} that
\begin{align*}
\mSigma_k &= \mW_0 \ostimes \mW_0 - (\mW_0 \ostimes \mW_0) \mQ(\mI_n \otimes \mS)\mG^{-1}(\mI_n \otimes \mS^\top)\mQ^\top (\mW_0 \ostimes \mW_0)\\
&= \mW_0 \ostimes \mW_0 - (2 \mW_0 - \mW_0 \mS (\mS^\top \mW_0 \mS)^{-1} \mS^\top \mW_0) \ostimes (\mW_0 \mS (\mS^\top \mW_0 \mS)^{-1} \mS^\top \mW_0)\\
&= (\mW_0 - \mW_0 \mS (\mS^\top \mW_0 \mS)^{-1} \mS^\top \mW_0) \ostimes (\mW_0 - \mW_0 \mS (\mS^\top \mW_0 \mS)^{-1} \mS^\top \mW_0)\\
&=\mW_0 (\mI_n-\mS\mU^\top) \ostimes \mW_0 (\mI_n-\mS\mU^\top).
\end{align*}
\end{proof}

\subsection{Matrix-variate Prior Construction}
From a practical point of view it is important to be able to construct a prior for \(\rmA\) and \(\rmH\) from an initial guess \(\vx_0\) for the solution. This reduces down to finding \(\mA_0\) and \(\mH_0\) symmetric positive definite, such that \(\mA_0 = \mH_0^{-1}\) and \(\vx_0 = \mH_0 \vb\) for the covariance class derived in \Cref{sec:prior_covariance_class}. We provide a computationally efficient construction of such a prior here.

\begin{proposition}
Let \(\vx_0\in \mathbb{R}^n\) and \(\vb \in \mathbb{R}^n \setminus \{0\}\). Assume \(\vx_0^\top \vb > 0\), then for \(\alpha < \frac{\vb^\top \vx_0}{\vb^\top \vb}\),
\begin{equation*}
\mH_0 = \alpha \mI + \frac{1}{(\vx_0 - \alpha \vb)^\top \vb} (\vx_0 - \alpha \vb)(\vx_0 - \alpha \vb)^\top
\end{equation*}
is symmetric positive definite and \(\mH_0 \vb = \vx_0\). Further it holds that
\begin{equation*}
\mA_0 = \mH_0^{-1} = \alpha^{-1} \mI - \frac{\alpha^{-1}}{(\vx_0 - \alpha \vb)^\top \vx_0}(\vx_0 - \alpha \vb)(\vx_0 - \alpha \vb)^\top.
\end{equation*}
If \(\vx_0^\top \vb < 0\) or \(\vx_0^\top \vb = 0\), then for \(\vx_1 = -\vx_0\) or \(\vx_1 = \frac{\vb^\top \vb}{\vb^\top \mA \vb}\vb\) respectively,  it holds that \(\norm{\vx_1 - \vx_*}_{\mA}^2 <  \norm{\vx_0 - \vx_*}_{\mA}^2\), i.e. \(\vx_1\) is a strictly better initialization than \(\vx_0\).
\end{proposition}
\begin{proof}
Let \(\mH_0\) as above. Then \(\mH_0 \vb = \alpha \vb + \vx_0 - \alpha \vb = \vx_0\). The second term of the sum in the form of \(\mH_0\) is of rank 1. Its non-zero eigenvalue is given by
\begin{equation*}
\lambda = \frac{1}{(\vx_0 - \alpha \vb)^\top \vb} (\vx_0 - \alpha \vb)^\top(\vx_0 - \alpha \vb) = \frac{1}{\vx_0^\top \vb - \alpha \vb^\top \vb} \norm{\vx_0 - \alpha \vb}_2^2 \geq 0
\end{equation*}
since by assumption \(\vx_0^\top \vb > 0\) and \(\alpha < \frac{\vb^\top \vx_0}{\vb^\top \vb}\). Now by Weyl's theorem it holds that \(\lambda_{\min}(\mA) + \lambda_{\min}(\mE) \leq \lambda_{\min}(\mA + \mE)\) and therefore \(\mH_0\) is positive definite. By the matrix inversion lemma we have for \(\gamma = \frac{\alpha^{-1}}{(\vx_0 - \alpha \vb)^\top \vb}\) that
\begin{align*}
\mA_0 &= \mH_0^{-1} = \alpha^{-1}(\mI - \frac{\gamma}{1 + \gamma \norm{\vx_0 - \alpha \vb}_2^2 } (\vx_0 - \alpha \vb)(\vx_0 - \alpha \vb)^\top)\\
&= \alpha^{-1} \mI - \frac{\alpha^{-2}}{(\vx_0 - \alpha b)^\top \vb + \alpha^{-1} \lVert \vx_0 - \alpha \vb \rVert_2^2}(\vx_0 - \alpha \vb)(\vx_0 - \alpha \vb)^\top\\
&= \alpha^{-1} \mI - \frac{\alpha^{-1}}{(\vx_0 - \alpha \vb)^\top \vx_0}(\vx_0 - \alpha \vb)(\vx_0 - \alpha \vb)^\top.
\end{align*}
Finally, we obtain
\begin{equation*}
\norm{\vx_0 - \vx_*}_\mA^2 = (\vx_0 - \mA^{-1}\vb)^\top \mA(\vx_0 - \mA^{-1}\vb) = \vx_0^\top \mA \vx_0 + \vb^\top \mA^{-1} \vb - 2 \vb^\top \vx_0.
\end{equation*}
Therefore if either \(\vx_0^\top \vb < 0\) or \(\vx_0^\top \vb = 0\), then \(\vx_1 = -\vx_0\) or \(\vx_1 = \frac{\vb^\top \vb}{\vb^\top \mA \vb}\vb\), respectively are closer to \(\vx_*\) in \(\mA\) norm by positive definiteness of \(\mA\). This concludes the proof.
\end{proof}

\subsection{Stopping Criteria}
\label{sec:stopping_criteria}

In addition to the classic stopping criteria \(\norm{\mA \vx_k - \vb}_2 \leq \max(\delta_{\textup{rtol}} \norm{\vb}_2, \delta_{\textup{atol}})\) it is natural from a probabilistic viewpoint to use the induced posterior covariance of \(\rvx\). Let \(\mM \in \mathbb{R}_{\textup{sym}}^{n \times n}\) be a positive-definite matrix, then by linearity and the cyclic property of the trace it holds that
\begin{align*}
\mathbb{E}_{\vx_*}[\norm{\vx_* - \mathbb{E}[\rvx]}_\mM^2] &= \mathbb{E}_{\vx_*}[(\vx_* - \mathbb{E}[\rvx])^\top \mM(\vx_* - \mathbb{E}[\rvx])]\\
&= \tr(\mathbb{E}_{\vx_*}[(\vx_* - \mathbb{E}[\rvx])^\top \mM(\vx_* - \mathbb{E}[\rvx])])\\
&= \mathbb{E}_{\vx_*}[\tr ((\vx_* - \mathbb{E}[\rvx])^\top \mM(\vx_* - \mathbb{E}[\rvx]))]\\
&= \mathbb{E}_{\vx_*}[\mM \tr ((\vx_* - \mathbb{E}[\rvx]) (\vx_* - \mathbb{E}[\rvx])^\top)]\\
&= \tr( \mM \mathbb{E}_{\vx_*}[ (\vx_* - \mathbb{E}[\rvx]) (\vx_* - \mathbb{E}[\rvx])^\top])\\
&= \tr(\mM(\operatorname{Cov}[\vx_* - \mathbb{E}[\rvx]] + (\mathbb{E}_{\vx_*}[\vx_*] - \mathbb{E}[\rvx])^\top(\mathbb{E}_{\vx_*}[\vx_*] - \mathbb{E}[\rvx])))\\
&= \tr(\mM \operatorname{Cov}[\vx_*]) +  \lVert \mathbb{E}_{\vx_*}[\vx_*] - \mathbb{E}[\rvx]\rVert_\mM^2.
\end{align*}
Assuming calibration holds, i.e. \(\vx_* \sim \mathcal{N}(\mathbb{E}[\rvx], \operatorname{Cov}[\rvx])\), we can bound the (relative) error by terminating when \(\tr(\mM\operatorname{Cov}[\rvx]) \leq \max(\delta_{\textup{rtol}} \norm{\vb}, \delta_{\textup{atol}})\) either in \(l_2\)-norm for \(\mM = \mI\) or in \(\mA\)-norm for \(\mM=\mA\).

We can efficiently evaluate the required \(\tr(\mM\operatorname{Cov}[\rvx])\) without ever forming \(\operatorname{Cov}[\rvx]\) in memory from already computed quantities. At iteration \(k\) we have \(\operatorname{Cov}[\rvx] = \operatorname{Cov}[\rmH \vb]= \frac{1}{2}(\mW_k^{\rmH} (\vb^{\top} \mW_k^{\rmH} \vb) + (\mW_k^{\rmH} \vb)(\mW_k^{\rmH} \vb)^\top)\) and therefore
\begin{equation*}
	\tr(\mM \operatorname{Cov}[\rvx]) = \frac{1}{2}\big((\vb^{\top} \mW_k^{\rmH} \vb) \tr(\mM \mW_k^{\rmH}) + (\mW_k^{\rmH} \vb)^\top \mM (\mW_k^{\rmH} \vb) \big).
\end{equation*}
Given the update for the covariance of the inverse view, we obtain the following recursion for its trace
\begin{equation*}
\tr(\mM \mW_k^{\rmH})= \tr(\mM \mW_{k-1}^{\rmH}) - \frac{1}{\vy_k^\top \mW_{k-1}^{\rmH} \vy_k} \tr((\mW_{k-1}^{\rmH} \vy_k)^\top \mM (\mW_{k-1}^{\rmH} \vy_k)) .
\end{equation*}
Computing the trace in this iterative fashion adds at most three matrix-vector products and three inner products for arbitrary \(\mM\) all other quantities are computed for the covariance update anyhow.

For our proposed covariance class \eqref{eqn:prior_cov_class} we obtain for \(\mM=\mI\) and \(\mPsi=\psi \mI\) that
\begin{align*}
\tr(\mW^{\rmH}_0) &= \tr(\mA_0^{-1}\mY(\mY^\top \mA_0^{-1} \mY)^{-1} \mY^\top \mA_0^{-1} + (\mI - \mY(\mY^\top \mY)^{-1}\mY^\top) \mPsi (\mI - \mY(\mY^\top \mY)^{-1}\mY^\top))\\
&= \tr((\mY^\top \mA_0^{-1} \mY)^{-1} \mY^\top \mA_0^{-1}\mA_0^{-1}\mY) + \psi \tr((\mI - \mY(\mY^\top \mY)^{-1}\mY^\top)(\mI - \mY(\mY^\top \mY)^{-1}\mY^\top))\\
&= \tr((\mY^\top \mA_0^{-1} \mY)^{-1} \mY^\top \mA_0^{-1}\mA_0^{-1}\mY) + \psi \tr(\mI - \mY(\mY^\top \mY)^{-1}\mY^\top)\\
&= \tr((\mY^\top \mA_0^{-1} \mY)^{-1} \mY^\top \mA_0^{-1}\mA_0^{-1}\mY) + \psi (n-k),
\end{align*}
which for a scalar prior mean \(\mA_0 = \alpha \mI\) reduces to \(\tr(\mW^{\rmH}_0) = \alpha^{-1}k + \psi (n-k)\).

\subsection{Implementation}
In order to maintain numerical stability when performing low rank updates to symmetric positive definite matrices, as is the case in \Cref{alg:problinsolve} for the mean and covariance estimates, it is advantageous use a representation based on the Cholesky decomposition. One can perform the rank-2 update for the mean estimate and the rank-1 downdate for the covariance in \Cref{thm:sym_gaussian_inference} in each iteration of the algorithm for their respective Cholesky factors instead (see also \citet{Seeger2008}). The rank-2 update can be seen as a combination of a rank-1 up- and downdate by recognizing that
\begin{equation*}
\vu \vv^\top + \vv\vu^\top = \frac{1}{2}((\vu+\vv)(\vu+\vv)^\top - (\vu-\vv)(\vu-\vv)^\top).
\end{equation*}
Similar updates arise in Quasi-Newton methods for the approximate (inverse) Hessian \cite{Nocedal2006}. Having Cholesky factors of the mean and covariance available has the additional advantage that downstream sampling or the evaluation of the probability density function is computationally cheap.

\section{Theoretical Properties: Proofs for \Cref{sec:theoretical_properties}}
\label{sec:theoretical_properties_proofs}

In this section we provide detailed proofs for the theoretical results on convergence and the connection of \Cref{alg:problinsolve} to the method of conjugate gradients. We restate each theorem here as a reference to the reader. We begin by proving an intermediate result giving an interpretation to the posterior mean of \(\rmA\) and \(\rmH\) at each step of the method.

\begin{proposition}[Subspace Equivalency]
Let \(\mA_k\) and \(\mH_k\) be the posterior means defined as in \Cref{sec:inference_framework} and assume \(\mA_0\) and \(\mH_0\) are symmetric. Then for \(1 \leq k \leq n\) it holds that
\begin{equation}
\label{eqn:subspace_equivalency}
\mA_k \mS = \mY \quad \textup{and}\quad \mH_k \mY = \mS,
\end{equation}
i.e. \(\mA_k\) and \(\mH_k\) act like \(\mA\) and \(\mA^{-1}\) on the spaces spanned by the actions \(\mS\), respectively the observations \(\mY\).
\end{proposition}
\begin{proof}
Since \(\mA_0\) and \(\mH_0\) are symmetric so are the expressions \(\mDelta_\rmA \mS\) and \(\mDelta_\rmH^{\top} \mY\). We have  that
\begin{align*}
\mA_k \mS &= (\mA_0 + \mDelta_\rmA \mU_{\rmA}^{\top} + \mU_{\rmA} \mDelta_\rmA^{\top} - \mU_{\rmA} \mS^{\top} \mDelta_\rmA \mU_{\rmA}^{\top}) \mS\\
&= \mA_0 \mS + \mDelta_\rmA \mI + \mU_{\rmA} \mDelta_{\rmA}^\top \mS - \mU_{\rmA} \mS^{\top} \mDelta_{\rmA} \mI\\
&=\mA_0 \mS + \mY - \mA_0 \mS\\
&= \mY.
\end{align*}
In the case of the inverse model we obtain
\begin{align*}
\mH_k \mY &= (\mH_0 + \mDelta_\rmH \mU_\rmH^{\top} + \mU_\rmH \mDelta_\rmH^{\top} - \mU_\rmH \mY^{\top}\mDelta_\rmH \mU_\rmH^{\top}) \mY\\
&= \mH_0 \mY + \mDelta_\rmH \mI + \mU_\rmH \mDelta_\rmH^{\top} \mY - \mU_\rmH \mY^{\top}\mDelta_\rmH \mI\\
&= \mH_0 \mY + \mS - \mH_0 \mY\\
&= \mS
\end{align*}
\end{proof}

\subsection{Conjugate Directions Method}

\begin{customthm}{1}[Conjugate Directions Method]
Given a prior \(p(\rmH)=\mathcal{N}(\rmH; \mH_0, \mW_0^{\rmH} \ostimes \mW_0^{\rmH})\) such that \(\mH_0, \mW_0^{\rmH} \in \mathbb{R}^{n \times n}_{\textup{sym}}\) positive definite, then actions \(\vs_i\) of \Cref{alg:problinsolve} are \(\mA\)-conjugate, i.e. for \(0 \leq i,j \leq k\) with \(i \neq j\) it holds that \(\vs_i^\top \mA \vs_j = 0\).
\end{customthm}
\begin{proof}
Since \(\mH_0\) is assumed to be symmetric, the form of the posterior mean in \Cref{sec:inference_framework} implies that \(\mH_k\) is symmetric for all \(1 \leq k \leq n\). Now conjugacy is shown by induction. To that end, first consider the base case \(k=2\). We have
\begin{align*}
\vs_2^\top \mA \vs_1 &= - \vr_1^\top \mH_1 \mA \vs_1 = - (\vr_0^\top + \alpha_1 \vy_1^\top)\mH_1 \mA \vs_1 = - \left(\vr_0^\top \mH_1 - \frac{\vs_1^\top \vr_0}{\vs_1^\top \vy_1} \vy_1^\top \mH_1 \right) \vy_1\\
&= -\vr_0^\top \vs_1 + \vs_1^\top \vr_0 = 0
\end{align*}
where we used \eqref{eqn:subspace_equivalency} and the definition of \(\alpha_i\) in \Cref{alg:problinsolve}. Now for the induction step, assume that \(\vs_i^\top \mA \vs_j = 0\) for all \(i \neq j\) such that \(1 \leq i,j \leq k\). We obtain for \(1 \leq j \leq k\) that
\begin{align*}
\vs_{k+1}^\top \mA \vs_j &= - \vr_k^\top \mH_k \mA \vs_j= - \bigg(\sum_{1 \leq l \leq k} \alpha_l \vy_l + \vr_0 \bigg)^\top \mH_k \vy_j= -\sum_{1 \leq l \leq k} \alpha_l \vy_l^\top \vs_j - \vr_0^\top \vs_j \\
&= - \alpha_j \vy_j^\top \vs_j - \vr_0^\top \vs_j= \vs_j^\top \vr_{j-1} - \vr_0^\top \vs_j= \vs_j^\top \bigg( \sum_{1 \leq l < j} \alpha_l \vy_l + \vr_0 \bigg) - \vr_0^\top \vs_j\\
&= \vs_j^\top \vr_0 - \vr_0^\top \vs_j= 0
\end{align*}
where we used the update equation of the residual \(\vr_i\) in \Cref{alg:problinsolve}, the definition of \(\alpha_i\), the induction hypothesis and \eqref{eqn:subspace_equivalency}. This proves the statement.
\end{proof}

\subsection{Relationship to the Conjugate Gradient Method}
\begin{customthm}{2}[Connection to the Conjugate Gradient Method]
Given a scalar prior mean \(\mA_0 = \mH_0^{-1}= \alpha \mI\) with \(\alpha > 0\), assume \eqref{eqn:hered_pos_def} and \eqref{eqn:post_mean_equiv} hold, then the iterates \(\bm{x}_i\) of \Cref{alg:problinsolve} are identical to the ones produced by the conjugate gradient method.
\end{customthm}
\begin{proof}
The proof outlined here is closely related to the proofs connecting Quasi-Newton methods to the conjugate gradient method \cite{Nazareth1979, Hennig2015}, but makes different assumptions on the prior distribution.

We begin by recognizing that the choice of step length \(\alpha_{i}\) in \Cref{alg:problinsolve} is identical to the one in the conjugate gradient method \cite{Nocedal2006}. Hence, it suffices to show that \(\vs_{i} \propto \vs_{i}^{\textup{CG}}\). \Cref{thm:conj_directions_method} established that \Cref{alg:problinsolve} is a conjugate directions method. Now by assumption \(\mA_0 = \alpha \mI\) and \(\mH_0 = \mA_0^{-1}\), therefore \(\vs_1 = -\alpha I \vr_0 \propto - \vr_0 =\vs_1^{\textup{CG}}\). It suffices show that \(\vs_{i}\) lies in the Krylov space \(\mathcal{K}_{i}(\mA, \vr_0)=\{\vr_0, \mA \vr_0, \dots, \mA^{i-1}\vr_0\}\) for all \(0 < i \leq n\). This completes the argument, since \(\mathcal{K}_{i}(\mA, \vr_0)\) is an \(i\)-dimensional subspace of \(\R^n\) and thus \(\mA\)-conjugacy uniquely determines the search directions up to scaling, as \(\mA\) is positive definite. 

To complete the proof we proceed as follows. The posterior mean of the inverse model \(\mH_{i-1}\) at step \(i-1\) maps an arbitrary vector \(\vv \in \mathbb{R}^n\) to \(\operatorname{span}( \mH_0 \vv, \mH_0 \mY_{1:i-1}, \mS_{1:i-1}, \mW_0^\rmH \mY_{1:i-1})\). This follows directly from its form in given in \Cref{sec:inference_framework}. By assumption \(\mH_0 = \mA_0^{-1} = \alpha^{-1} \mI\), therefore using \eqref{eqn:hered_pos_def} and \eqref{eqn:post_mean_equiv} we have \(\operatorname{span}(\mW_0^\rmH \mY_{1:i-1}) = \operatorname{span}(\mY_{1:i-1})\). This implies \(\mH_{i-1}\) maps to \(\operatorname{span}( \vv, \mS_{1:i-1}, \mY_{1:i-1})\) and thus \(\vs_{i} \in \operatorname{span}( \vr_{i-1}, \mS_{1:i-1}, \mY_{1:i-1} )\). We will now show that \(\operatorname{span}( \vr_{i-1}, \mS_{1:i-1}, \mY_{1:i-1} ) \subset \mathcal{K}_{i}(\mA, \vr_0)\) by induction, completing the argument. 

We begin with the base case. Since \(\mH_0\) is assumed to be scalar, we have \(\vs_1 \propto \vr_0 \in \mathcal{K}_0(\mA, \vr_0)\) and therefore \(\vy_1 = \mA \vs_1\) and \(\vr_1 = \vr_0 + \alpha_1 \vy_1\) are in \(\mathcal{K}_1(\mA, \vr_0)\). For the induction step assume \(\operatorname{span}( \vr_{i-1}, \mS_{1:i-1}, \mY_{1:i-1}) \subset \mathcal{K}_{i}(\mA, \vr_0)\). The definition of the policy of \Cref{alg:problinsolve} gives
\begin{equation*}
\vs_{i} = - \mathbb{E}[\rmH]\vr_{i-1} \propto \mH_{i-1} \vr_{i-1} \in  \operatorname{span}( \vr_{i-1}, \mS_{1:i-1}, \mY_{1:i-1} ) \subset \mathcal{K}_{i}(\mA, \vr_0),
\end{equation*}
where we used the induction hypothesis. This implies that \(\vy_{i} = \mA \vs_{i} \in \mathcal{K}_{i+1}(\mA, \vr_0)\) and \(\vr_{i} = \vr_{i-1} + \alpha_{i}\vy_{i} \in \mathcal{K}_{i+1}(\mA, \vr_0)\) by the definition of the Krylov space. Therefore, \(\operatorname{span}(\vr_{i}, \mS_{1:i}, \mY_{1:i}) \subset \mathcal{K}_{i+1}(\mA, \vr_0)\). This completes the proof.
\end{proof}

\section{Prior Covariance Class: Proofs for \Cref{sec:prior_covariance_class}}
\label{sec:proofs_prior_covariance_class}

\subsection{Hereditary Positive-Definiteness}

\begin{customprop}{1}[Hereditary Positive Definiteness \cite{Dennis1977, Hennig2013}]
Let \(\mA_0 \in \mathbb{R}^{n \times n}_{\textup{sym}}\) be positive definite. Assume the actions \(\mS\) are \(\mA\)-conjugate and \(\mW_0^\rmA \mS = \mY\), then for \(i \in \{0, \dots, k - 1\}\) it holds that \(\mA_{i+1}\) is symmetric positive definite.
\end{customprop}
\begin{proof}
This is shown in \citet{Hennig2013}. We give an identical proof in our notation as a reference to the reader. By Theorem 7.5 in \citet{Dennis1977} it holds that if \(\mA_i\) is positive definite and \(\vs_{i+1}^\top \mW_i^\rmA \vs_{i+1} \neq 0\), then \(\mA_{i+1}\) is positive definite if and only if \(\det (\mA_{i+1}) > 0\). By the matrix determinant lemma and the recursive formulation of the posterior we have
\begin{align*}
\det(\mA_{i+1}) &= \det(\mA_i) \bigg(\frac{1}{(\vs_{i+1}^\top \mW_i^\rmA \vs_{i+1})^2}\big((\vy_{i+1}^\top \mA_i^{-1} \mW_i^\rmA \vs_{i+1})^2\\ 
&\quad- (\vy_{i+1}^\top \mA_i^{-1} \vy_{i+1})(\vs_{i+1}^\top \mW_i^\rmA  \mA_i^{-1} \mW_i^\rmA  \vs_{i+1})+ (\vs_{i+1}^\top \mW_i^\rmA  \mA_i^{-1} \mW_i^\rmA \vs_{i+1} )(\vy_{i+1}^\top \vs_{i+1})\big)\bigg)
\end{align*}
Hence it suffices to show that
\begin{align*}
0 < (\vy_{i+1}^\top \mA_i^{-1} \mW_i^\rmA \vs_{i+1})^2 &- (\vy_{i+1}^\top \mA_i^{-1} \vy_{i+1})(\vs_{i+1}^\top \mW_i^\rmA \mA_i^{-1} \mW_i^\rmA \vs_{i+1})\\
 &+ (\vs_{i+1}^\top \mW_i^\rmA \mA_i^{-1} \mW_i^\rmA \vs_{i+1} )(\vy_{i+1}^\top \vs_{i+1}),
\end{align*}
which simplifies to
\begin{equation*}
\vy_{i+1}^\top \mA_i^{-1} \vy_{i+1} - \frac{(\vy_{i+1}^\top \mA_i^{-1} \mW_i^\rmA \vs_{i+1})^2}{\vs_{i+1}^\top \mW_i^\rmA \mA_i^{-1} \mW_i^\rmA \vs_{i+1} }< \vy_{i+1}^\top \vs_{i+1}
\end{equation*}
Now by \(\mW_0^\rmA \mS = \mY\), we have \(\mW_i^\rmA \vs_{i+1} = \mW_0^\rmA \vs_{i+1} = \vy_{i+1}\) and the above reduces to
\begin{equation*}
0 < \vs_{i+1}^\top \mA \vs_{i+1},
\end{equation*}
which is fulfilled by the assumption that \(\mA\) is positive definite. Thus \(\mA_{i+1}\) is positive definite. Symmetry follows immediately from the form of the posterior mean.
\end{proof}

\subsection{Posterior Correspondence}
\begin{customdef}{1}
Let \(\mA_i\) and \(\mH_i\) be the means of \(\rmA\) and \(\rmH\) at step \(i\). We say a prior induces \emph{posterior correspondence} if
\begin{equation}
\label{eqn:posterior_mean_correspondence}
\mA_i^{-1} = \mH_i
\end{equation}
for all steps \(0 \leq i \leq k\) of the solver. If only
\begin{equation}
\label{eqn:weak_posterior_mean_correspondence}
\mA_i^{-1}\mY = \mH_i \mY,
\end{equation}
we say that \emph{weak posterior correspondence} holds.
\end{customdef}

\subsubsection{Matrix-variate Normal Prior}
We begin by establishing posterior correspondence in the case of general matrix-variate normal priors, i.e. the inference setting detailed in \Cref{cor:asymmetric_gaussian_inference}. We begin by proving a general non-constructive condition and close with a sufficient condition for correspondence with limits the possible choices of covariance factors to a specific class.

\begin{lemma}[General Correspondence]
\label{lem:general_correspondence}
Let \(1 \leq k \leq n\), \(\mW_0^\rmA , \mW_0^\rmH \) symmetric positive-definite and assume \(\mA_0^{-1} = \mH_0\), then \eqref{eqn:posterior_mean_correspondence} holds if and only if
\begin{equation}
\label{eqn:general_correspondence}
0 = (\mA\mS-\mA_0 \mS)\left[ (\mS^{\top}\mW_0^\rmA  \mA_0^{-1}\mA\mS)^{-1}\mS^{\top}\mW_0^\rmA \mA_0^{-1} - (\mS^{\top}\mA^{\top}\mW_0^\rmH \mA\mS)^{-1}\mS^{\top}\mA^{\top}\mW_0^\rmH \right].
\end{equation}
\end{lemma}
\begin{proof}
By the matrix inversion lemma we have
\begin{align*}
0 &= \mA_k^{-1} - \mH_k\\
&= \left(\mA_0  + (\mY-\mA_0 \mS)(\mS^{\top}\mW_0^\rmA \mS)^{-1}\mS^{\top}\mW_0^\rmA  \right)^{-1} - \mH_0-(\mS-\mH_0\mY)(\mY^{\top}\mW_0^\rmH \mY)^{-1}\mY^{\top}\mW_0^\rmH \\
&= \mA_0^{-1} - \mA_0^{-1}(\mY-\mA_0 \mS)(\mS^{\top}\mW_0^\rmA \mS + \mS^{\top}\mW_0^\rmA \mA_0^{-1}(\mY-\mA_0 \mS))^{-1}\mS^{\top}\mW_0^\rmA \mA_0^{-1}\\
&\quad - \mA_0^{-1}-\mA_0^{-1}(\mA_0 \mS-\mY)(\mY^{\top}\mW_0^\rmH \mY)^{-1}\mY^{\top}\mW_0^\rmH \\
&= -\mA_0^{-1}(\mY-\mA_0 \mS)\left[(\mS^{\top}\mW_0^\rmA \mA_0^{-1}\mY)^{-1}\mS^{\top}\mW_0^\rmA \mA_0^{-1} -  (\mY^{\top}\mW_0^\rmH \mY)^{-1}\mY^{\top}\mW_0^\rmH \right],
\end{align*}
where we used the assumption \(\mH_0 = \mA_0^{-1}\). Left-multiplying with \(-\mA_0 \) and using \(\mY=\mA\mS\) completes the proof.
\end{proof}

\begin{corollary}[Correspondence at Convergence]
Let \(k=n\), \(\mH_0 = \mA_0^{-1}\) and assume \(\mS\) has full rank, i.e. the linear solver has performed \(n\) linearly independent actions, then \eqref{eqn:posterior_mean_correspondence} holds for any symmetric positive-definite choice of \(\mW_0^\rmA \) and \(\mW_0^\rmH \).
\end{corollary}
\begin{proof}
By assumption, \(\mS^{\top}\mW_0^\rmA \mA_0^{-1}\) and \(\mS^{\top}\mA^{\top}\mW_0^\rmH \) are invertible. Then by \Cref{lem:general_correspondence} the correspondence condition \eqref{eqn:posterior_mean_correspondence} holds.
\end{proof}

\begin{theorem}[Sufficient Condition for Correspondence]
Let \(1 \leq k \leq n\) arbitrary and assume \(\mH_0 = \mA_0^{-1}\). Assume $\mW_0^\mA, \mA_0, \mW_0^\mH$ satisfy
\begin{equation}
\label{eqn:sufficient_condition_posterior_correspondence}
0 = \mS^{\top}(\mW_0^\rmA \mA_0^{-1} - \mA^{\top}\mW_0^\rmH )
\end{equation}
or equivalently let \(\mB_{\langle \mS \rangle^\perp} \in \mathbb{R}^{n \times k}\) be a basis of the orthogonal space  \(\langle \mS \rangle^{\perp}\) spanned by the actions. For \(\mPhi \in \mathbb{R}^{(n-k) \times n}\) arbitrary, if
\begin{equation}
\label{eqn:sufficient_condition_correspondence_direct_expression}
\mW_0^\rmH  = \mA^{-\top}(\mW_0^\rmA \mA_0^{-1}-\mB_{\langle \mS \rangle^\perp}\mPhi)
\end{equation}
and the commutation relations
\begin{align}
\label{eqn:mean_commutation}[\mA_0 , \mA] &= \bm{0}\\
\label{eqn:cov_factor_commutation}[\mW_0^\rmA , \mA] &= \bm{0}\\
\label{eqn:deg_freedom_commutation}[\mB_{\langle \mS \rangle^\perp}\mPhi, \mA] &= \bm{0}
\end{align}
are fulfilled, then \(\mW_0^\rmH \) is symmetric and \eqref{eqn:posterior_mean_correspondence} holds.
\end{theorem}
\begin{proof}
By assumption \(\mW_0^\rmA \) is symmetric positive-definite and \eqref{eqn:sufficient_condition_posterior_correspondence} is equivalent to \(\mS^{\top}\mW_0^\rmA \mA_0^{-1} = \mS^{\top}\mA^{\top}\mW_0^\rmH \), which implies \eqref{eqn:general_correspondence}. Now, assumption \eqref{eqn:sufficient_condition_posterior_correspondence} is equivalent to columns of the difference \(\mW_0^\rmA \mA_0^{-1} - \mA^{\top}\mW_0^\rmH \) lying in \(L\), i.e. we can choose a basis \(\mB_{\langle \mS \rangle^\perp}\) and coefficient matrix \(\mPhi\) such that
\begin{equation*}
\mW_0^\rmA \mA_0^{-1} - \mA^{\top}\mW_0^\rmH  = \mB_{\langle \mS \rangle^\perp}\mPhi.
\end{equation*}
Rearranging the above gives \eqref{eqn:sufficient_condition_correspondence_direct_expression}. With the commutation relations and
\begin{equation*}
[\mA,\mB]=\bm{0} \iff [\mA^{-1}, \mB]=\bm{0} \iff [\mA, \mB^{-1}]=\bm{0} \iff [\mA^{-1}, \mB^{-1}]=\bm{0}
\end{equation*}
it holds that
\begin{equation*}
(\mW_0^\rmH)^{\top} = \mW_0^\rmA \mA_0^{-1}\mA^{-1}-\mB_{\langle \mS \rangle^\perp}\mPhi \mA^{-1}= \mA^{-\top}\mW_0^\rmA \mA_0^{-1} - \mA^{-\top}\mB_{\langle \mS \rangle^\perp}\mPhi= \mW_0^\rmH
\end{equation*}
hence \(\mW_0^\rmH \) is symmetric. Finally, by \Cref{lem:general_correspondence} posterior mean correspondence \eqref{eqn:posterior_mean_correspondence} holds.
\end{proof}

If we want to ensure correspondence for all iterations, \eqref{eqn:deg_freedom_commutation} is trivially satisfied. The question now becomes what form can \(\mA_0 \) and \(\mW_0^\rmA \) take in order to ensure symmetric \(\mW_0^\rmH \). This comes down to finding matrices which commute with \(\mA\).

\begin{lemma}[Commuting Matrices of a Symmetric Matrix]
\label{lem:commuting_matrices}
Let \(r \in \mathbb{N}\), \(\mM \in \mathbb{R}^{n \times n}\) and \(\mA \in \mathbb{R}^{n \times n}\) symmetric. Assume \(\mM\) has the form
\begin{equation*}
\mM = \mathfrak{p}_r(\mA) = \sum_{i=0}^r c_i \mA^i
\end{equation*}
for a set of coefficients \(c_i \in \mathbb{R}\), then \(\mM\) and \(\mA\) commute. If \(\mA\) has \(n\) distinct eigenvalues, \(\mM\) is diagonalizable and \([\mM, \mA] = \bm{0}\), then
\begin{equation*}
\mM = \mathfrak{p}_{n-1}(\mA),
\end{equation*}
i.e. \(\mM\) is a polynomial in \(\mA\) of degree at most \(n-1\).
\end{lemma}
\begin{proof}
The first result follows immediately since
\begin{equation*}
\mW_0^\rmA \mA = \mathfrak{p}_r(\mA)\mA = \sum_{i=0}^r c_i \mA^{i+1} = \mA \mathfrak{p}_r(\mA) = \mA \mW_0^\rmA .
\end{equation*}
Assume now that \(\mA\) has \(n\) distinct eigenvalues \(\lambda_0 , \dots, \lambda_{n-1}\), \(\mM\) is diagonalizable and \(\mM\) and \(\mA\) commute. Now, if and only if \([\mA,\mM]=0\), then \(\mA\) and \(\mM\) are simultaneously diagonalizable by Theorem 5.2 in \citet{Conrad2008}, i.e. we can find a common basis in which both \(\mA\) and \(\mM\) are represented by diagonal matrices. Hence, the set of matrices commuting with \(\mA\) forms an \(n\)-dimensional subspace \(\mathcal{U}_n \subset \mathbb{R}^{n \times n}\). Now, by the first part of this proof \(\{\mI, \mA, \dots, \mA^{n-1}\} \subset \mathcal{U}_n\). It remains to be shown, that this set forms a basis of \(\mathcal{U}_n\). By isomorphism of finite dimensional vector spaces this is equivalent to proving that
\begin{equation*}
\{\vb_0, \vb_1, \dots, \vb_{n-1}\} \coloneqq \left\{
\begin{pmatrix}
1\\ \vdots \\ 1
\end{pmatrix},
\begin{pmatrix}
\lambda_0 \\ \vdots \\ \lambda_{n-1}
\end{pmatrix}, \dots,
\begin{pmatrix}
\lambda_0^{n-1}\\ \vdots \\ \lambda_{n-1}^{n-1}
\end{pmatrix}
\right\}
\end{equation*}
forms a basis of \(\mathbb{R}^n\). It suffices to show that all \(\vb_i\) are independent. Assume the contrary, then
\(
\sum_{i=0}^{n-1} \alpha_i \vb_i = 0
\)
for some \(\alpha_0 , \dots, \alpha_{n-1} \in \mathbb{R}\), such that not all \(\alpha_i=0\). This implies that the polynomial \(\sum_{i=0}^{n-1}\alpha_i x^i\) has \(n\) zeros \(\lambda_0 , \dots, \lambda_{n-1}\). This contradicts the fundamental theorem of algebra, concluding the proof.
\end{proof}

The above suggests that tractable choices of \(\mA_0 \) and \(\mW_0^\rmA \) for the non-symmetric matrix-variate prior, which imply symmetric \(\mW_0^\rmH \), are of polynomial form in \(\mA\). 

\begin{example}[Posterior Correspondence Covariance Class]
    Tractable choices of the prior parameters in the \(\rmA\) view, which satisfy posterior correspondence and the commutation relations are for example
    \begin{equation*}
    \mA_0  = c_0 \mI \qquad \textup{and} \qquad \mW_0^\rmA = \sum_{i=1}^{n-1}c_i \mA^i,
    \end{equation*}
 	where \(\mH_0 = \mA_0^{-1}\) with \(c_i \in \mathbb{R}\). Motivated by \(\tr(\mA) \overset{!}{=} \tr(\mA_0 )\) an initial choice could be \(c_0 = n^{-1}\tr(\mA)\).
\end{example}

Finally, note that in practice we do not actually require \(\mW_0^\rmA\). We only ever need access to \(\mW_0^\rmA \mS\).

\subsubsection{Symmetric Matrix-variate Normal Prior}

We now turn to the symmetric model, which we assumed throughout the paper, given in \Cref{thm:sym_gaussian_inference}. We prove \Cref{thm:weak_post_correspondence}, the main result of this section demonstrating \emph{weak posterior correspondence} for the symmetric Kronecker covariance, by employing the matrix inversion lemma for the posterior mean \(\mA_k\). We begin by establishing a set of technical lemmata first, which mainly expand terms appearing during matrix block inversion.

\begin{lemma}[Symmetric Posterior Inverse]
\label{lem:symm_posterior_inverse}
Under the assumptions of \Cref{thm:sym_gaussian_inference}, the inverse of the posterior mean is given by
\begin{equation*}
\mA_k^{-1} = \mA_0^{-1} - \mA_0^{-1} \begin{bmatrix} \mU_{\rmA}  & \mV_\rmA  \end{bmatrix}
\begin{bmatrix} \mU_{\rmA} ^\top \mA_0^{-1}\mU_{\rmA}  & \mI+ \mU_{\rmA} ^\top \mA_0^{-1}\mV_\rmA  \\ \mI+\mV_\rmA ^\top \mA_0^{-1}\mU_{\rmA}  & \mV_\rmA ^{\top}\mA_0^{-1}\mV_\rmA \end{bmatrix}^{-1} \begin{bmatrix} \mU^{\top}_\rmA\\ \mV_\rmA ^\top \end{bmatrix}\mA_0^{-1}
\end{equation*}
where
\begin{align*}
\mU_{\rmA}  &\coloneqq \mW_0^\rmA \mS(\mS^{\top}\mW_0^\rmA \mS)^{-1} \in \mathbb{R}^{n \times k},\\
\mV_\rmA  &\coloneqq (\mI-\frac{1}{2}\mU_\rmA \mS^\top)(\mY-\mA_0 \mS) = (\mI-\frac{1}{2}\mU_\rmA \mS^\top)\mDelta_\rmA  \in \mathbb{R}^{n \times k}.
\end{align*}
\end{lemma}
\begin{proof}
We rewrite the rank-2 update in \Cref{sec:inference_framework} as follows
\begin{equation*}
\mA_k = \mA_0 + \mU_\rmA \mV_\rmA^\top + \mV_\rmA \mU_\rmA^\top =
\mA_0 + \begin{bmatrix} \mU_\rmA  & \mV_\rmA  \end{bmatrix} \begin{bmatrix} \bm{0} & \mI\\ \mI & \bm{0} \end{bmatrix} \begin{bmatrix} \mU^{\top}_\rmA\\ \mV_\rmA ^\top \end{bmatrix}.
\end{equation*}
Then the statement follows directly from the matrix inversion lemma.
\end{proof}

Next, we expand the terms inside the blocks of the matrix to be inverted in \Cref{lem:symm_posterior_inverse}. This leads to the following lemma.

\begin{lemma}
\label{lem:symm_posterior_inverse_blockterms}
Given the assumptions of \Cref{thm:sym_gaussian_inference}, let \(\mW_0^\rmA \) and \(\mA_0 \) be symmetric and assume \eqref{eqn:post_mean_equiv} and \eqref{eqn:hered_pos_def} hold. Define
\begin{align*}
\mLambda &= \mS^\top \mW_0^\rmA  \mS \\
\mPi &= \mS^\top \mW_0^\rmA  \mA_0^{-1}\mDelta_\rmA ,
\end{align*}
then \(\mLambda \in \mathbb{R}^{m \times m}\) and \(\mLambda + \mPi \in \mathbb{R}^{m \times m}\) are symmetric and invertible and we obtain
\begingroup
\allowdisplaybreaks
\begin{align}
\mLambda + \mPi &= \mS^\top \mW_0^\rmA  \mA_0^{-1}\mA\mS =\mS^\top \mA \mA_0^{-1}\mA\mS = \mS^\top \mA \mW_0^\rmH  \mA\mS\\
\mPi &= \mDelta_\rmA ^{\top} \mA_0^{-1}\mA\mS\\
\mU_\rmA ^\top \mA_0^{-1}\mDelta_\rmA  &= \mLambda^{-1}\mPi\\
\mDelta_\rmA ^{\top}\mS &= \mS^\top \mDelta_\rmA \\
\mU_\rmA  &= \mA\mS\mLambda^{-1}\\
\mU_\rmA ^\top \mA_0^{-1} \mU_\rmA  &= \mLambda^{-1}(\mLambda + \mPi)\mLambda^{-1}\\
\mI + \mU_\rmA ^{\top}\mA_0^{-1}\mV_\rmA  &= \mLambda ^{-1}(\mLambda + \mPi)(\mI - \frac{1}{2}\mLambda^{-1}\mS^\top\mDelta_\rmA )\\
\mI + \mV_\rmA ^{\top}\mA_0^{-1}\mU_\rmA  &=(\mI - \frac{1}{2}\mDelta_\rmA ^\top \mS \mLambda^{-1})(\mLambda + \mPi)\mLambda ^{-1}\\
\mV_\rmA ^{\top}\mA_0^{-1}\mV_\rmA  &= \mPi - \frac{1}{2}\big((\mLambda + \mPi)\mLambda^{-1}\mS^\top \mDelta_\rmA  + \mDelta_\rmA ^\top \mS \mLambda^{-1}(\mLambda + \mPi) \big)\\
 &\qquad + \frac{1}{4}\mDelta_\rmA ^\top \mS \mLambda ^{-1}(\mLambda + \mPi)\mLambda^{-1} \mS^\top \mDelta_\rmA
\end{align}
\endgroup
\end{lemma}
\begin{proof}
We begin by proving that \(\mLambda\) and \(\mLambda + \mPi\) are symmetric and invertible. We have by Sylvester's rank inequality that \(\mLambda\) is invertible. For symmetric \(\mW_0^\rmA \), \(\mLambda\) is symmetric by definition. We have that
\begin{align*}
\mLambda + \mPi &= \mS^\top \mW_0^\rmA  \mS + \mS^\top \mW_0^\rmA  \mA_0^{-1} (\mA\mS - \mA_0 \mS) = \mS^\top \mW_0^\rmA  \mA_0^{-1}\mA\mS =\mS^\top \mA \mA_0^{-1}\mA\mS\\
&= \mS^\top \mW_0^\rmA  \mA_0^{-1}\mA\mS = \mS^\top \mA \mW_0^\rmH  \mA\mS\\
\end{align*}
Thus, by Sylvester's rank inequality \(\mLambda + \mPi\) is invertible. Given symmetric \(\mA_0 \), it is symmetric. Further, it holds that
\begin{align*}
\mPi &= \mLambda + \mPi - \mLambda = \mS^\top \mA \mA_0^{-1} \mA \mS - \mS^\top \mA \mS = \mDelta_\rmA ^{\top} \mA_0^{-1}\mA\mS\\
\mU_\rmA ^\top \mA_0^{-1}\mDelta_\rmA  &= (\mS^{\top}\mW_0^\rmA \mS)^{-1}\mS^{\top}\mW_0^\rmA  \mA_0^{-1} \mDelta_\rmA  = \mLambda^{-1}\mPi\\
\mDelta_\rmA ^\top \mS &=(\mA\mS - \mA_0 \mS)^\top \mS = \mS^\top \mA\mS - \mS^\top \mA_0  \mS\\
\mU_\rmA  &= \mW_0^\rmA \mS(\mS^{\top}\mW_0^\rmA \mS)^{-1} = \mA\mS \mLambda^{-1}\\
\mU_\rmA ^\top \mA_0^{-1} \mU_\rmA  &= \mLambda^{-1} \mS^\top \mA \mA_0^{-1} \mA\mS \mLambda^{-1} = \mLambda^{-1}(\mLambda + \mPi)\mLambda^{-1}\\
\mI + \mU_\rmA ^{\top}\mA_0^{-1}\mV_\rmA  &= \mI + \mLambda^{-1} \mS^\top \mA \mA_0^{-1} (\mI-\frac{1}{2}\mU_\rmA \mS^\top)\mDelta_\rmA  = \mI + \mLambda^{-1} \mS^\top \mA \mA_0^{-1} (\mI-\frac{1}{2}\mA\mS\mLambda^{-1}\mS^\top)\mDelta_\rmA \\
&=\mI + \mLambda^{-1} \mS^\top \mA \mA_0^{-1}(\mA\mS - \mA_0 \mS)-\frac{1}{2}\mLambda^{-1} \mS^\top \mA \mA_0^{-1}\mA\mS\mLambda^{-1}\mS^\top \mDelta_\rmA \\
&= \mLambda^{-1} (\mLambda + \mPi) -\frac{1}{2}\mLambda^{-1} (\mLambda + \mPi)\mLambda^{-1}\mS^\top \mDelta_\rmA  = \mLambda ^{-1}(\mLambda + \mPi)(\mI - \frac{1}{2}\mLambda^{-1}\mS^\top\mDelta_\rmA )\\
\mI + \mV_\rmA ^{\top}\mA_0^{-1}\mU_\rmA  &= (\mI + \mU_\rmA ^{\top}\mA_0^{-1}\mV_\rmA )^\top = (\mLambda ^{-1}(\mLambda + \mPi)(\mI - \frac{1}{2}\mLambda^{-1}\mS^\top\mDelta_\rmA ))^\top\\
&=(\mI - \frac{1}{2}\mDelta_\rmA ^\top \mS \mLambda^{-1})(\mLambda + \mPi)\mLambda ^{-1},
\end{align*}
where we used that \(\mLambda\) and \(\mLambda + \mPi\) are symmetric. Finally, we have that
\begingroup
\allowdisplaybreaks
\begin{align*}
\mV_\rmA ^{\top}\mA_0^{-1}\mV_\rmA  &= \mDelta_\rmA ^{\top}(\mI - \frac{1}{2}\mS \mU_\rmA ^\top)\mA_0^{-1}(\mI-\frac{1}{2}\mU_\rmA \mS^\top)\mDelta_\rmA \\
&= \mDelta_\rmA ^{\top}(\mI - \frac{1}{2}\mS\mLambda^{-1} \mS^\top \mA)\mA_0^{-1}(\mI-\frac{1}{2}\mA\mS\mLambda^{-1} \mS^\top)\mDelta_\rmA \\
&= \mDelta_\rmA ^{\top}\mA_0^{-1}(\mI-\frac{1}{2}\mA\mS\mLambda^{-1} \mS^\top)\mDelta_\rmA  - \frac{1}{2}\mDelta_\rmA ^{\top}\mS\mLambda^{-1} \mS^\top \mA \mA_0^{-1}(\mI-\frac{1}{2}\mA\mS\mLambda^{-1} \mS^\top)\mDelta_\rmA \\
&=(\mS^\top \mA \mA_0^{-1} - \mS^\top)\big(\mI - \frac{1}{2}\mA\mS\mLambda^{-1} \mS^\top)\mDelta_\rmA  - \frac{1}{2} \mDelta_\rmA ^\top \mS \mLambda^{-1} \mS^\top \mA \mA_0^{-1} \mDelta_\rmA\\
&\qquad+ \frac{1}{4} \mDelta_\rmA ^\top \mS \mLambda^{-1} \mS^\top \mA \mA_0^{-1}\mA\mS\mLambda^{-1}\mS^\top \mDelta_\rmA \\
&= \mS^\top \mA \mA_0^{-1}\mDelta_\rmA  - \mS^\top \mDelta_\rmA  - \frac{1}{2} \mS^\top \mA \mA_0^{-1}\mA\mS\mLambda^{-1}\mS^\top \mDelta_\rmA  + \frac{1}{2} \mS^\top \mA \mS \mLambda^{-1} \mS^\top \mDelta_\rmA\\
&\qquad-\frac{1}{2} \mDelta_\rmA ^\top \mS \mLambda^{-1} \mS^\top \mA \mA_0^{-1} \mDelta_\rmA  + \frac{1}{4} \mDelta_\rmA ^\top \mS \mLambda^{-1} \mS^\top \mA \mA_0^{-1}\mA\mS\mLambda^{-1}\mS^\top \mDelta_\rmA\\
&= \mS^\top \mA \mA_0^{-1}\mA\mS - \mS^\top \mA \mS - \mS^\top \mDelta_\rmA  - \frac{1}{2}(\mLambda + \mPi) \mLambda^{-1} \mS^\top \mDelta_\rmA  + \frac{1}{2} \mS^\top \mDelta_\rmA\\
&\qquad-\frac{1}{2} \mDelta_\rmA ^\top \mS \mLambda^{-1} \mS^\top \mA \mA_0^{-1} \mDelta_\rmA  + \frac{1}{4} \mDelta_\rmA ^\top \mS \mLambda^{-1} \mS^\top \mA \mA_0^{-1}\mA\mS\mLambda^{-1}\mS^\top \mDelta_\rmA\\
&= \mPi - \frac{1}{2} \mS^\top \mDelta_\rmA  - \frac{1}{2} (\mLambda + \mPi)\mLambda^{-1} \mS^\top \mDelta_\rmA  - \frac{1}{2} \mDelta_\rmA ^\top \mS \mLambda^{-1} \mS^\top \mA \mA_0^{-1}(\mA\mS - \mA_0 \mS)\\
&\qquad+ \frac{1}{4} \mDelta_\rmA ^\top \mS \mLambda^{-1} \mS^\top \mA \mA_0^{-1}\mA\mS\mLambda^{-1}\mS^\top \mDelta_\rmA\\
&= \mPi - \frac{1}{2} \mS^\top \mDelta_\rmA  - \frac{1}{2} (\mLambda + \mPi)\mLambda^{-1} \mS^\top \mDelta_\rmA  - \frac{1}{2} \mDelta_\rmA ^\top \mS \mLambda^{-1}(\mLambda + \mPi) + \frac{1}{2} \mDelta_\rmA ^\top \mS \mLambda^{-1} \mLambda\\
&\qquad+ \frac{1}{4} \mDelta_\rmA ^\top \mS \mLambda^{-1} \mS^\top \mA \mA_0^{-1}\mA\mS\mLambda^{-1}\mS^\top \mDelta_\rmA\\
&= \mPi - \frac{1}{2}\big((\mLambda + \mPi)\mLambda^{-1}\mS^\top \mDelta_\rmA  + \mDelta_\rmA ^\top \mS \mLambda^{-1}(\mLambda + \mPi) \big) + \frac{1}{4} \mDelta_\rmA ^\top \mS \mLambda^{-1} (\mLambda + \mPi) \mLambda^{-1}\mS^\top \mDelta_\rmA ,
\end{align*}
\endgroup
where we dropped some of the terms temporarily for clarity of exposition.
\end{proof}
We will now use these intermediate results to perform block inversion on the \(2k \times 2k\) matrix to be inverted in \Cref{lem:symm_posterior_inverse}.

\begin{lemma}
\label{lem:symm_posterior_inverse_blockterms_inverse}
Given the assumptions of \Cref{thm:sym_gaussian_inference}, additionally assume \eqref{eqn:hered_pos_def} and \eqref{eqn:post_mean_equiv} hold. Let
\begin{equation*}
\mT = \begin{bmatrix}
\mT_{11} & \mT_{12}\\\mT_{21} & \mT_{22}
\end{bmatrix} =\begin{bmatrix} \mU_\rmA ^\top \mA_0^{-1}\mU_\rmA  & \mI+ \mU_\rmA ^\top \mA_0^{-1}\mV_\rmA  \\ \mI+\mV_\rmA ^\top \mA_0^{-1}\mU_\rmA  & \mV_\rmA ^{\top}\mA_0^{-1}\mV_\rmA \end{bmatrix}^{-1},
\end{equation*}
then the block matrices \(\mT_{ij} \in \mathbb{R}^{m \times m}\) are given by
\begin{align*}
\mT_{11} &= \mLambda (\mLambda + \mPi)^{-1}\mLambda -  (\mI - \frac{1}{2}\mS^\top\mDelta_\rmA  \mLambda^{-1})(\mI - \frac{1}{2} \mLambda^{-1} \mDelta_\rmA ^\top \mS )\\
\mT_{12} &= (\mI - \frac{1}{2}\mS^\top\mDelta_\rmA  \mLambda^{-1})\\
\mT_{21} &= \mT_{12}^\top =(\mI - \frac{1}{2}\mLambda^{-1} \mDelta_\rmA ^\top \mS)\\
\mT_{22} &= - \mLambda^{-1}.
\end{align*}
\end{lemma}
\begin{proof}
Let
\begin{equation*}
\mK = \mT^{-1} = \begin{bmatrix} \mU_\rmA ^\top \mA_0^{-1}\mU_\rmA  & \mI+ \mU_\rmA ^\top \mA_0^{-1}\mV_\rmA  \\ \mI+\mV_\rmA ^\top \mA_0^{-1}\mU_\rmA  & \mV_\rmA ^{\top}\mA_0^{-1}\mV_\rmA \end{bmatrix},
\end{equation*}
then the inverse of the Schur complement \(\mD=\mK/(\mU_\rmA ^\top \mA_0^{-1}\mU_\rmA )\) is given by
\begin{align*}
\mD^{-1} &= (\mK_{22} - \mK_{21}\mK_{11}^{-1}\mK_{12})^{-1}\\
& =\big(\mV_\rmA ^{\top}\mA_0^{-1}\mV_\rmA  - (\mI + \mV_\rmA ^{\top}\mA_0^{-1}\mU_\rmA )(\mU_\rmA ^\top \mA_0^{-1} \mU_\rmA )^{-1}(\mI + \mU_\rmA ^{\top}\mA_0^{-1}\mV_\rmA )\big)^{-1}\\
&=\big(\mV_\rmA ^{\top}\mA_0^{-1}\mV_\rmA  - (\mI - \frac{1}{2}\mDelta_\rmA ^\top \mS \mLambda^{-1})(\mLambda + \mPi)(\mI - \frac{1}{2}\mLambda^{-1} \mS^\top \mDelta_\rmA ) \big)^{-1}\\
&=\big(\mV_\rmA ^{\top}\mA_0^{-1}\mV_\rmA  - (\mLambda - \frac{1}{2}\mDelta_\rmA ^\top \mS )\mLambda^{-1}(\mLambda + \mPi)\mLambda^{-1}(\mLambda - \frac{1}{2} \mS^\top \mDelta_\rmA ) \big)^{-1}\\
&= \big( \mV_\rmA ^{\top}\mA_0^{-1}\mV_\rmA  - (\mLambda + \mPi) + \frac{1}{2}\big(\mDelta_\rmA ^\top \mS \mLambda^{-1} (\mLambda + \mPi ) + (\mLambda + \mPi ) \mLambda^{-1}\mS^\top \mDelta_\rmA  \big)\\
&\qquad- \frac{1}{4} \mDelta_\rmA ^\top \mS \mLambda^{-1} (\mLambda + \mPi) \mLambda^{-1} \mS^\top \mDelta_\rmA  \big)^{-1}\\
&= (\mPi - \mLambda - \mPi)^{-1}\\
&= - \mLambda^{-1},
\end{align*}
where we used \Cref{lem:symm_posterior_inverse_blockterms}. By block matrix inversion and again with \Cref{lem:symm_posterior_inverse_blockterms} we obtain
\begin{align*}
\mT_{11} &=(\mU_\rmA ^\top \mA_0^{-1} \mU_\rmA )^{-1} + (\mU_\rmA ^\top \mA_0^{-1} \mU_\rmA )^{-1}(\mI + \mU_\rmA ^{\top}\mA_0^{-1}\mV_\rmA ) \mD^{-1} (\mI + \mV_\rmA ^{\top}\mA_0^{-1}\mU_\rmA )(\mU_\rmA ^\top \mA_0^{-1} \mU_\rmA )^{-1}\\
&= \mLambda (\mLambda + \mPi)^{-1}\mLambda + \mLambda (\mI - \frac{1}{2}\mLambda^{-1}\mS^\top\mDelta_\rmA )\mD^{-1}(\mI - \frac{1}{2}\mDelta_\rmA ^\top \mS \mLambda^{-1}) \mLambda\\
&= \mLambda (\mLambda + \mPi)^{-1}\mLambda +  (\mLambda - \frac{1}{2}\mS^\top\mDelta_\rmA )\mD^{-1}(\mLambda - \frac{1}{2}\mDelta_\rmA ^\top \mS )
\intertext{as well as}
\mT_{12} &=-(\mU_\rmA ^\top \mA_0^{-1} \mU_\rmA )^{-1}(\mI + \mU_\rmA ^{\top}\mA_0^{-1}\mV_\rmA ) \mD^{-1}  \\
&= -\mLambda (\mLambda + \mPi)^{-1}\mLambda \mLambda ^{-1}(\mLambda + \mPi)(\mI - \frac{1}{2}\mLambda^{-1}\mS^\top\mDelta_\rmA )\mD^{-1}\\
&= - (\mLambda - \frac{1}{2}\mS^\top\mDelta_\rmA )\mD^{-1}\\
\mT_{21} &= \mT_{12}^\top = -\mD^{-\top}(\mLambda - \frac{1}{2}\mDelta_\rmA ^\top \mS)
\end{align*}
and finally \(\mT_{22} = \mD^{-1} = - \mLambda^{-1}.\)
\end{proof}

\begin{lemma}
\label{lem:post_mean_inverse_blockinversion}
Given the assumptions of \Cref{thm:sym_gaussian_inference}, additionally assume \eqref{eqn:hered_pos_def} and \eqref{eqn:post_mean_equiv} hold. Let
\begin{equation*}
\mF = \mA_0^{-1} \begin{bmatrix} \mU_\rmA  & \mV_\rmA  \end{bmatrix}
\begin{bmatrix} \mT_{11} & \mT_{12} \\ \mT_{21} & \mT_{22} \end{bmatrix} \begin{bmatrix} \mU^{\top}_\rmA\\ \mV_\rmA ^\top \end{bmatrix}\mA_0^{-1},
\end{equation*}
where \(\mT\) is chosen as in \Cref{lem:symm_posterior_inverse_blockterms_inverse}, then if \(\mS^\top \mA \mS = \mI\), we have
\begin{equation*}
\mF = \mA_0^{-1}\mA\mS(\mI+\mPi)^{-1}\mS^\top \mA \mA_0^{-1} - \mS\mS^\top.
\end{equation*}
\end{lemma}
\begin{proof}
By expanding the quadratic and using \Cref{lem:symm_posterior_inverse_blockterms_inverse}, we obtain the terms
\begingroup
\allowdisplaybreaks
\begin{align*}
\mF_{11} &\coloneqq \mA_0^{-1}\mU_\rmA \mT_{11}\mU_\rmA ^\top \mA_0^{-1}\\
 &= \mA_0^{-1}\mU_\rmA  \mLambda (\mLambda + \mPi)^{-1} \mLambda \mU_\rmA ^\top \mA_0^{-1} - \mA_0^{-1} \mU_\rmA  (\mI - \frac{1}{2}\mS^\top \mDelta_\rmA  \mLambda^{-1}) (\mI - \frac{1}{2}\mLambda^{-1} \mDelta_\rmA ^\top \mS) \mU_\rmA ^\top \mA_0^{-1}\\
 &= \mA_0^{-1} \mA\mS (\mLambda + \mPi)^{-1} \mS^\top \mA \mA_0^{-1} - \mA_0^{-1} \mA\mS \mLambda^{-1} (\mI - \frac{1}{2}\mS^\top \mDelta_\rmA  \mLambda^{-1}) (\mI - \frac{1}{2}\mLambda^{-1} \mDelta_\rmA ^\top \mS) \mLambda^{-1} \mS^\top \mA \mA_0^{-1}\\
 &= \mA_0^{-1} \mA\mS (\mLambda + \mPi)^{-1} \mS^\top \mA \mA_0^{-1} - \mA_0^{-1}\mA\mS \mLambda^{-2} \mS^\top \mA \mA_0^{-1}\\
 &\qquad+ \frac{1}{2}\mA_0^{-1}\mA\mS \mLambda^{-1}(\mS^\top \mDelta_\rmA  \mLambda^{-1} + \mLambda^{-1} \mDelta_\rmA ^\top \mS) \mLambda^{-1} \mS^\top \mA \mA_0^{-1}\\
  &\qquad- \frac{1}{4} \mA_0^{-1} \mA\mS \mLambda^{-1} \mS^{\top} \mDelta_\rmA  \mLambda^{-2} \mDelta_\rmA ^\top \mS \mLambda^{-1} \mS^\top \mA \mA_0^{-1}\\
\mF_{12} &\coloneqq \mA_0^{-1}\mU_\rmA \mT_{12}\mV_\rmA ^\top \mA_0^{-1}\\
&= \mA_0^{-1}\mU_\rmA (\mI - \frac{1}{2}\mS^{\top}\mDelta_\rmA  \mLambda^{-1})\mV_\rmA ^\top \mA_0^{-1}\\
&= \mA_0^{-1}\mA\mS\mLambda^{-1}(\mI - \frac{1}{2}\mS^{\top}\mDelta_\rmA  \mLambda^{-1})\mDelta_\rmA ^\top( \mI - \frac{1}{2}\mS \mU_\rmA ^\top) \mA_0^{-1} \\
&= \mA_0^{-1}\mA\mS\mLambda^{-1}(\mI - \frac{1}{2}\mS^{\top}\mDelta_\rmA  \mLambda^{-1})\mDelta_\rmA ^\top( \mI - \frac{1}{2}\mS \mLambda^{-1}\mS^\top \mA) \mA_0^{-1}\\
&= \mA_0^{-1}\mA\mS \mLambda^{-1} \mDelta_\rmA ^{\top} \mA_0^{-1} - \frac{1}{2} \mA_0^{-1} \mA\mS \mLambda^{-1} (\mS^\top \mDelta_\rmA  \mLambda^{-1}\mDelta_\rmA ^\top + \mDelta_\rmA ^\top \mS \mLambda^{-1}\mS^\top \mA) \mA_0^{-1}\\
&\qquad+ \frac{1}{4} \mA_0^{-1}\mA\mS\mLambda^{-1}\mS^\top \mDelta_\rmA  \mLambda^{-1} \mDelta_\rmA ^\top \mS \mLambda^{-1} \mS^\top \mA \mA_0^{-1}\\
\mF_{21} &\coloneqq \mF_{12}^\top = \mA_0^{-1}(\mI - \frac{1}{2}\mA\mS \mLambda^{-1} \mS^{\top})\mDelta_\rmA (\mI - \frac{1}{2}\mLambda^{-1} \mDelta_\rmA ^{\top}\mS) \mLambda^{-1} \mS \mA \mA_0^{-1}\\
&=\mA_0^{-1} \mDelta_\rmA  \mLambda^{-1} \mS^\top \mA \mA_0^{-1} - \frac{1}{2} \mA_0^{-1} (\mDelta_\rmA  \mLambda^{-1} \mDelta_\rmA ^\top \mS + \mA \mS \mLambda^{-1} \mS^\top \mDelta_\rmA )\mLambda^{-1} \mS^\top \mA \mA_0^{-1} \\
&\qquad+ \frac{1}{4} \mA_0^{-1} \mA \mS \mLambda^{-1} \mS^\top \mDelta_\rmA  \mLambda^{-1} \mDelta_\rmA ^\top \mS\mLambda^{-1}\mS^\top \mA \mA_0^{-1}\\
\mF_{22} &\coloneqq \mA_0^{-1}\mV_\rmA T_{22}\mV_\rmA ^\top \mA_0^{-1}\\
&= - \mA_0^{-1}(\mI-\frac{1}{2}\mU_\rmA \mS^\top)\mDelta_\rmA  \mLambda^{-1} \mDelta_\rmA ^\top (\mI-\frac{1}{2}\mS \mU_\rmA ^\top)\mA_0^{-1}\\
&= - \mA_0^{-1}(\mI-\frac{1}{2}\mA\mS \mLambda^{-1} \mS^\top)\mDelta_\rmA  \mLambda^{-1} \mDelta_\rmA ^\top (\mI-\frac{1}{2}\mS\mLambda^{-1} \mS^\top \mA)\mA_0^{-1}\\
&= -\mA_0^{-1} \mDelta_\rmA  \mLambda^{-1} \mDelta_\rmA ^\top \mA_0^{-1} + \frac{1}{2} \mA_0^{-1}(\mA\mS \mLambda^{-1}\mS^\top \mDelta_\rmA  \mLambda^{-1} \mDelta_\rmA ^\top + \mDelta_\rmA  \mLambda^{-1} \mDelta_\rmA ^\top \mS \mLambda^{-1} \mS^\top \mA)\mA_0^{-1} \\
&\qquad- \frac{1}{4} \mA_0^{-1} \mA\mS \mLambda^{-1} \mS^{\top} \mDelta_\rmA  \mLambda^{-1} \mDelta_\rmA ^{\top}\mS \mLambda^{-1} \mS^\top \mA \mA_0^{-1}
\end{align*}
\endgroup
Assuming \(\mS^\top \mA \mS = \mI\), it holds that
\begingroup
\allowdisplaybreaks
\begin{align*}
\mF_{11} &= \mA_0^{-1}\mA\mS(\mI+\mPi)^{-1}\mS^\top \mA \mA_0^{-1} - \mA_0^{-1}\mA\mS\mS^\top  \mA\mA_0^{-1} + \frac{1}{2}\mA_0^{-1}\mA\mS(\mS^\top \mDelta_\rmA  + \mDelta_\rmA ^\top \mS) \mS^\top \mA \mA_0^{-1} \\
&\qquad- \frac{1}{4} \mA_0^{-1} \mA\mS\mS^\top \mDelta_\rmA  \mDelta_\rmA ^\top \mS \mS^\top \mA \mA_0^{-1}\\
\mF_{12} &= \mA_0^{-1} \mA\mS\mS^\top \mA \mA_0^{-1} - \mA_0^{-1}\mA\mS\mS^\top - \frac{1}{2} \mA_0^{-1}\mA\mS(\mS^\top \mDelta_\rmA  \mDelta_\rmA ^\top + \mDelta_\rmA ^\top \mS\mS^\top \mA)\mA_0^{-1}\\
&\qquad+\frac{1}{4}\mA_0^{-1}\mA\mS\mS^\top \mDelta_\rmA  \mDelta_\rmA ^\top \mS\mS^\top \mA \mA_0^{-1}\\
\mF_{21} &= \mA_0^{-1} \mA\mS\mS^\top \mA \mA_0^{-1} - \mS\mS^\top \mA\mA_0^{-1} - \frac{1}{2}\mA_0^{-1}(\mDelta_\rmA  \mDelta_\rmA ^\top \mS + \mA \mS\mS^\top \mDelta_\rmA )\mS^\top \mA \mA_0^{-1}\\
&\qquad+\frac{1}{4}\mA_0^{-1}\mA\mS\mS^\top \mDelta_\rmA  \mDelta_\rmA ^\top \mS\mS^\top \mA \mA_0^{-1}\\
\mF_{22} &= \mA_0^{-1}\mDelta_\rmA  \mS^\top -\mA_0^{-1}\mDelta_\rmA  \mS^\top \mA \mA_0^{-1} + \frac{1}{2}(\mA\mS \mS^\top \mDelta_\rmA  \mDelta_\rmA ^\top + \mDelta_\rmA  \mDelta_\rmA ^\top \mS \mS^\top \mA)\mA_0^{-1}\\
&\qquad- \frac{1}{4} \mA_0^{-1}\mA\mS\mS^\top \mDelta_\rmA  \mDelta_\rmA ^\top \mS\mS^\top \mA \mA_0^{-1},
\end{align*}
\endgroup
which leads to
\begin{align*}
\mF_{11} + \mF_{12} &= \mA_0^{-1}\mA\mS(\mI+\mPi)^{-1}\mS^\top \mA \mA_0^{-1} - \mA_0^{-1}\mA\mS\mS^\top + \frac{1}{2} \mA_0^{-1}\mA\mS(\mS^{\top} \mDelta_\rmA  \mS^\top \mA - \mS^{\top} \mDelta_\rmA  \mDelta_\rmA ^{\top} )\mA_0^{-1}\\
\mF_{21} + \mF_{22} &= \mA_0^{-1} \mDelta_\rmA  \mS^{\top} + \frac{1}{2} \mA_0^{-1}\mA\mS (\mS^\top \mDelta_\rmA  \mDelta_\rmA ^\top - \mS^\top \mDelta_\rmA  \mS^\top \mA)\mA_0^{-1}\\
&= \mA_0^{-1} \mA\mS \mS^{\top} - \mS \mS^{\top} + \frac{1}{2} \mA_0^{-1}\mA\mS (\mS^\top \mDelta_\rmA  \mDelta_\rmA ^\top - \mS^\top \mDelta_\rmA  \mS^\top \mA)\mA_0^{-1}.
\end{align*}
Finally, adding up the individual terms we obtain
\begin{equation*}
\mF=\mF_{11} + \mF_{12} + \mF_{21} + \mF_{22} = \mA_0^{-1}\mA\mS(\mI+\mPi)^{-1}\mS^\top \mA \mA_0^{-1} - \mS\mS^\top.
\end{equation*}
\end{proof}

\begin{customthm}{2}[Weak Posterior Correspondence]
Let \(\mW_0^\rmH \in \mathbb{R}^{n \times n}_{\textup{sym}}\) be positive definite. Assume \(\mH_0 = \mA_0^{-1}\), and that $\mW_0^\rmA, \mA_0, \mW_0^\rmH$ satisfy \eqref{eqn:hered_pos_def} and \eqref{eqn:post_mean_equiv}, then weak posterior correspondence holds for the symmetric Kronecker covariance.
\end{customthm}

\begin{proof}
First note that without loss of generality \(\mS^\top \mA \mS = \mI\), i.e. only the direction of the action matters in \Cref{alg:problinsolve} not its magnitude. This can be seen from the forms of \(\mA_k\) and \(\mH_k\) in \Cref{sec:inference_framework}. Any positive factor \(\alpha > 0\) of \(\vs_k\) cancels in the update expressions. Expanding the right hand side we have using \eqref{eqn:subspace_equivalency}, that \(\mH_k \mY =\mS\). Then by \Cref{lem:symm_posterior_inverse},  \Cref{lem:post_mean_inverse_blockinversion} and \(\mS^\top \mA\mS = \mI\), the left hand side evaluates to
\begin{align*}
\mA_k^{-1}\mY&= (\mA_0^{-1} - \mF)\mY \\
&= (\mA_0^{-1} - \mA_0^{-1}\mA\mS(\mI + \mPi)^{-1}\mS^\top \mA \mA_0^{-1} + \mS\mS^\top)\mA\mS\\
&= \mA_0^{-1}\mA\mS - \mA_0^{-1}\mA\mS + \mS\\
&= \mS\\
&= \mH_k \mY.
\end{align*}
This concludes the proof.
\end{proof}

This theorem shows that for a certain choice of symmetric matrix-variate normal prior the estimated inverse of the matrix \(\mH_k\) corresponds to the inverse of the estimated matrix \(\mA_k^{-1}\). It also shows that both act like \(\mA^{-1}\) on the space spanned by \(\mY\), consistent with the interpretation of the two being the best guess for the inverse \(\mA^{-1}\).

\section{Galerkin's Method for PDEs}
\label{sec:pde_discretization}

In the spirit of applying machine learning in the sciences \cite{Carleo2019}, we briefly outlined an application of \Cref{alg:problinsolve} to the solution of partial differential equations in \Cref{sec:experiments}. As an example we considered the Dirichlet problem for the Poisson equation given by
\begin{equation}
\label{eqn:dirichlet_problem}
\begin{cases}
-\Delta u(x,y) = f(x,y) &(x,y) \in \operatorname{int}\Omega\\
u(x,y) = u_{\partial \Omega}(x,y) &(x,y) \in \partial \Omega
\end{cases}
\end{equation}
where \(\Omega\) is a connected open region with sufficiently regular boundary and \(u_{\partial \Omega} : \partial \Omega \rightarrow \mathbb{R}\) defines the boundary conditions. The corresponding weak solution of \eqref{eqn:dirichlet_problem} is given by \(u \in V\) such that for all test functions \(v \in V\)
\begin{equation}
\label{eqn:weak_form_dirichlet}
a(u,v) \coloneqq \int_{\Omega} \nabla u \cdot \nabla v \, dx = \int_{\Omega} fv \, dx \eqqcolon f(v),
\end{equation}
where \(a(\cdot, \cdot)\) is a bilinear form. Next, one derives the \emph{Galerkin equation} by choosing a finite-dimensional subspace \(V_{\mathsmaller{\square}} \subset V\) and corresponding basis \(e_1^{\mathsmaller{\square}}, \dots, e_n^{\mathsmaller{\square}}\). Then \eqref{eqn:weak_form_dirichlet} reduces to finding \(u \in V_{\mathsmaller{\square}}\) such that for all \(i \in \{1, \dots, n\}\) it holds that \(a(u, e_i^{\mathsmaller{\square}}) = \sum_{j=1}^n  u_j a(e_j^{\mathsmaller{\square}}, e_i^{\mathsmaller{\square}})= f(e_i^{\mathsmaller{\square}})\) which is a linear system \(\mA \vu=\vf\) with the entries of the Gram matrix given by \(\mA_{ij} = a(e_j^{\mathsmaller{\square}}, e_i^{\mathsmaller{\square}})\) and \(\evf_i = f(e_i^{\mathsmaller{\square}})\).

\subsection{Operator View}

The operator view provides another motivation for placing a distribution over the matrix \(\mA\) of a linear system. When approximating the solution to a PDE, as we do here, then solution-based inference for linear systems \citep{Cockayne2019, Bartels2019} can be viewed as placing a Gaussian process prior over the solution \(u : \Omega \rightarrow \mathbb{R}\) \citep{Girolami2020}. The matrix-based approach \cite{Hennig2015} instead can be interpreted as placing a Gaussian measure \cite{Bogachev1998} on the infinite-dimensional space of the differential operator instead. This induces a Gaussian distribution on the Gram matrix \(\rmA\) modelling the uncertainty about the actions of the (discretized) differential operator.

\begin{definition}[Infinite-dimensional Gaussian Measures \cite{Bogachev1998}]
\label{def:infinitedim_gaussian_measure}
Let \(W\) be a topological vector space with Borel probability measure \(\mu\), then \(\mu\) is Gaussian, iff for each continuous linear functional \(f \in W^*\), the pushforward \(\mu \circ f^{-1}\) is a Gaussian measure on \(\mathbb{R}\), i.e. \(f\) is a Gaussian random variable on \((W, \mathcal{B}_W, \mu)\).
\end{definition}
This definition  and further detail on Gaussian measures in infinite-dimensional spaces can be found in the book by \citet{Bogachev1998}. We now model the differential operator as a random variable on the space of bounded linear operators and show that this induces a distribution on the Gram matrix arising from discretization via Galerkin's method.

\begin{theorem}[Gaussian Measures on the Space of Bounded Linear Operators]
Let \(V\) be a Hilbert space and let \(W=B(V,V)\) be the space of bounded linear operators from \(V\) to \(V\) with Borel probability measure \(\mu\) and let \(\mathsf{A}\) be a Gaussian random variable on \((W, \mathcal{B}_W, \mu)\). Consider the operator equation
\begin{equation*}
\mathsf{A} u=\mathsf{f}
\end{equation*}
and let \(a: V \times V \rightarrow \mathbb{R}, (u, v) \mapsto \langle \mathsf{A} u , v \rangle_V = \langle \mathsf{f} , v \rangle_V\) be its corresponding bilinear form. Let \(V_{\mathsmaller{\square}}\) be an \(n\)-dimensional subspace of \(V\), then the resulting Gram matrix \(\rmA \in \mathbb{R}^{n \times n}\) is matrix-variate Gaussian.
\end{theorem}
\begin{proof}
Since \(V\) is Banach, so is \(W\). Define the functional \(a_W : W \rightarrow \mathbb{R}\) given by \(a_W(\mathsf{A}, u,v) = a(u,v)\) for fixed \(u, v \in V\). The map \(a_W(\cdot, u,v)\) is linear by linearity of the inner product and bounded since using the Cauchy-Schwarz inequality, it holds that
\begin{equation*}
\abs{a_W(\mathsf{A}, u,v)} = \abs{\langle \mathsf{A} u, v \rangle_{V}} \leq  \norm{\mathsf{A} u}_V \norm{v}_V \leq \norm{\mathsf{A}}_W \norm{u}_V \norm{v}_V=C \norm{\mathsf{A}}_W.
\end{equation*}
Therefore \(a_W(\cdot, u,v) \in W^*\) for all \(u, v \in V\). By \Cref{def:infinitedim_gaussian_measure} of a Gaussian measure the push forward \(\mu \circ a_W^{-1}\) is a Gaussian measure on \(\mathbb{R}\) for all \(u, v \in V\), in particular also for a basis \(\{v_i\}_{i=1}^n\) of \(V_{\mathsmaller{\square}}\). Therefore the Gram matrix \(\rmA\) given by \(\rmA_{ij} = a(v_i, v_j) = a_W(\mathsf{A}, v_i, v_j)\) is matrix-variate Gaussian since its components are Gaussian.
\end{proof}

\begin{remark}
The Laplacian \(\Delta : H^2(\Omega) \rightarrow L^2(\Omega)\) is a bounded linear operator on the Sobolev space \(H^2(\Omega)\). Note, that in general differential operators are in fact \emph{not bounded}. Hence, the simple argument above does not generalize to arbitrary differential operators.
\end{remark}

\begin{remark}
If the bilinear form \(a\) in addition to being continuous is also weakly coercive, then by the Lax-Milgram theorem the operator equation has a unique solution. A symmetric and weakly coercive operator implies a symmetric positive-definite Gram matrix.
\end{remark}

\subsection{Discretization Refinement}

The linear system \(\mA \vu = \vf\) arises from discretizing \eqref{eqn:dirichlet_problem} using Galerkin's method on a given mesh \({\mathsmaller{\square}}\) defined via a finite-dimensional subspace \(V_{\mathsmaller{\square}} \subset V\) such that \(\vu \in V_{\mathsmaller{\square}}\). By solving this problem using a probabilistic linear solver we obtain a posterior distribution over the inverse \(\mH\) of the discretized differential operator \(\mA\). Our goal is to leverage the obtained information about the solution on the coarse mesh to extrapolate to a refined discretization, similar in spirit to multi-grid methods \cite{Wesseling2004}. This approach can be seen as an instance of transfer learning and could be used for adaptive probabilistic mesh refinement strategies based on the uncertainty about the solution in a certain region of the mesh.

Consider a fine mesh \(\mathsmaller{\boxplus}\) given by \(V_{\mathsmaller{\boxplus}}\), where \(n_{\mathsmaller{{\mathsmaller{\boxplus}}}} = \operatorname{dim}(V_{\mathsmaller{\boxplus}}) > \operatorname{dim}(V_{\mathsmaller{\square}}) = n_{\mathsmaller{\square}}\) such that \(V_{\mathsmaller{\square}} \subset V_{{\mathsmaller{\boxplus}}} \subset V\). We would like to transfer information from solving the problem on the coarse mesh \(V_{{\mathsmaller{\square}}}\) to the solution of the discretized PDE on the fine mesh \(V_{{\mathsmaller{\boxplus}}}\). To do so we compute the predictive distribution on the fine mesh, given the belief over the inverse differential operator on the coarse mesh, i.e.

\begin{equation*}
p(\rmH_{{\mathsmaller{\boxplus}}}) = \int p(\rmH_{{\mathsmaller{\boxplus}}} \mid \rmH_{{\mathsmaller{\square}}}) p(\rmH_{{\mathsmaller{\square}}}) \, d \rmH_{{\mathsmaller{\square}}}.
\end{equation*}

Define the \emph{prolongation operator} \(\mP : \mathbb{R}^{n_{{\mathsmaller{\square}}}} \rightarrow \mathbb{R}^{n_{{\mathsmaller{\boxplus}}}}\) given by \(\mP_{ij}=\langle \ve_i^{\mathsmaller{\boxplus}}, \ve_j^{\mathsmaller{\square}} \rangle\) satisfying \(\mP^\top \mP = \mI \in \mathbb{R}^{n_{{\mathsmaller{\square}}} \times n_{{\mathsmaller{\square}}}}\), implying it is injective. The distribution over the inverse operator on the fine mesh given the inverse operator on the coarse mesh is given by
\begin{equation}
\label{eqn:pde_likelihood}
p(\rmH_{{\mathsmaller{\boxplus}}} \mid \rmH_{{\mathsmaller{\square}}}) =  \mathcal{N}(\rmH_{{\mathsmaller{\boxplus}}}; \mP \rmH_{{\mathsmaller{\square}}} \mP^\top, \mLambda)
\end{equation}
where \(\mLambda \in \mathbb{R}^{n_{\mathsmaller{\boxplus}} \times n_{\mathsmaller{\boxplus}}}_{\text{sym}}\) positive definite models the numerical uncertainty induced by the coarser discretization. This corresponds to the assumption that solving the problem on a coarser grid approximates the solution on a fine grid projected to the coarse grid. 

Now assume we have a posterior distribution over the inverse differential operator on the coarse grid from a solve of the coarse problem using \Cref{alg:problinsolve}, given by
\begin{equation*}
p(\rmH_{{\mathsmaller{\square}}})= \mathcal{N}(\rmH_{{\mathsmaller{\square}}};\mH_{{\mathsmaller{\square}}}^k, \mW_{{\mathsmaller{\square}}}^k \ostimes \mW_{{\mathsmaller{\square}}}^k).
\end{equation*}
The projection in \eqref{eqn:pde_likelihood} is a linear map, since by the characteristic property of the Kronecker product \eqref{eqn:kronecker_charprop} we have
\begin{equation*}
\operatorname{svec}(\mP \rmH_{{\mathsmaller{\square}}} \mP^\top) = \mQ(\mP \otimes \mP)\mQ^\top \operatorname{svec}(\rmH_{{\mathsmaller{\square}}}).
\end{equation*}
Therefore by \Cref{thm:gaussian_inference} the predictive distribution is also closed-form and Gaussian.

\begin{proposition}[Predictive Distribution on Fine Mesh]
\label{prop:predictive_fine_mesh}
Let \(p(\rmH_{{\mathsmaller{\square}}})= \mathcal{N}(\rmH_{{\mathsmaller{\square}}};\mH_{{\mathsmaller{\square}}}^k, \mW_{{\mathsmaller{\square}}}^k \ostimes \mW_{{\mathsmaller{\square}}}^k)\) be a prior on \(\rmH_{{\mathsmaller{\square}}}\) and assume a likelihood of the form \eqref{eqn:pde_likelihood}. Then the predictive distribution is given by \(p(\rmH_{{\mathsmaller{\boxplus}}}) = \mathcal{N}(\rmH_{{\mathsmaller{\boxplus}}}; \mH_{{\mathsmaller{\boxplus}}}^0, \mSigma_{{\mathsmaller{\boxplus}}}^0)\), where
\begin{align*}
\mH_{{\mathsmaller{\boxplus}}}^0 &= \mP \mH_{{\mathsmaller{\square}}}^k \mP^\top,\\
\mSigma_{{\mathsmaller{\boxplus}}}^0 &= \mP\mW_{{\mathsmaller{\square}}}^k \mP^\top \ostimes \mP \mW_{{\mathsmaller{\square}}}^k \mP^\top + \mLambda. 
\end{align*}
\end{proposition}

\begin{proof}
By \Cref{thm:gaussian_inference} we obtain for the mean and covariance of the predictive distribution
\begin{align*}
\mH_{{\mathsmaller{\boxplus}}}^0 &= \mP \mH_{{\mathsmaller{\square}}}^k \mP^\top\\
\mSigma_{{\mathsmaller{\boxplus}}}^0  &= \mQ(\mP \otimes \mP) \mQ^\top (\mW_{{\mathsmaller{\square}}}^k \ostimes \mW_{{\mathsmaller{\square}}}^k) \mQ (\mP^\top \otimes \mP^\top)\mQ^\top + \mLambda\\
&= \frac{1}{2}\mQ(\mP \mW_{{\mathsmaller{\square}}}^k \mP^\top \otimes \mP \mW_{{\mathsmaller{\square}}}^k \mP^\top + \mP \mW_{{\mathsmaller{\square}}}^k \mP^\top \boxtimes \mP \mW_{{\mathsmaller{\square}}}^k \mP^\top)\mQ^\top + \mLambda\\
&= \mP\mW_{{\mathsmaller{\square}}}^k \mP^\top \ostimes \mP \mW_{{\mathsmaller{\square}}}^k \mP^\top + \mLambda
\end{align*}
where we used \eqref{eqn:kron_symkron_kron} and the symmetry of \(\mW_{{\mathsmaller{\square}}}^k\).
\end{proof}

For general \(\mLambda\) the covariance of the predictive distribution does not have symmetric Kronecker form, making its use as a prior for a new solve on the fine mesh challenging. We aim to exploit structural assumptions on \(\mLambda\) and results on nearest Kronecker products to a sum of Kronecker products to remedy this shortcoming in the future.

\end{document}

%% file: preamble_paper.tex
\usepackage[utf8]{inputenc} 
\usepackage[T1]{fontenc}    
\usepackage{microtype}      

\usepackage{xcolor}

\usepackage{enumitem}
\setitemize{noitemsep,topsep=0pt,parsep=0pt,partopsep=0pt} 

\usepackage{amsmath}
\usepackage{amssymb}
\usepackage{amsfonts}       
\usepackage{nicefrac}       
\usepackage{mathtools}
\usepackage{relsize}
\usepackage{bm}
\usepackage{physics}

\usepackage{etoolbox}
\usepackage{pgfplots}
\usepgfplotslibrary{groupplots}
\usepackage{tikz}
\usetikzlibrary{tikzmark}
\usetikzlibrary{calc}

\usetikzlibrary{external}
\tikzset{external/force remake=false}
\tikzsetexternalprefix{figures/external/}

\newcommand{\cmark}{}%
\DeclareRobustCommand{\cmark}{%
  \tikz\fill[scale=0.4, color=black!30!green]
  (0,.35) -- (.25,0) -- (1,.7) -- (.25,.15) -- cycle;%
}
\newcommand{\xmark}{}%
\DeclareRobustCommand{\xmark}{%
  \tikz [x=1.4ex,y=1.4ex,line width=.2ex, color=red] \draw (0,0) -- (1,1) (0,1) -- (1,0);
}
\newcommand{\umark}{{\color{orange}\(\thicksim\)}}

\usepackage{booktabs}
\usepackage{siunitx}
\usepackage{csvsimple}

\usepackage{graphicx}
\usepackage{pgfplots}
\pgfplotsset{compat=1.15}
\usepackage{caption}
\usepackage{subcaption}

\usepackage{algorithm}
\usepackage[noend]{algpseudocode}
\algrenewcommand{\algorithmiccomment}[1]{\hfill {\textcolor{darkgray}{\# #1}}}  
\algrenewcommand\algorithmicindent{2.5em}   
\algrenewcommand\alglinenumber[1]{\tiny #1} 

\makeatletter
\newcommand\fs@booktabsruled{%
  \def\@fs@cfont{\bfseries\strut}\let\@fs@capt\floatc@ruled
  \def\@fs@pre{\hrule height\heavyrulewidth depth0pt \kern\belowrulesep}%
  \def\@fs@mid{\kern\aboverulesep\hrule height\lightrulewidth\kern\belowrulesep}%
  \def\@fs@post{\kern\aboverulesep\hrule height\heavyrulewidth\relax}%
  \let\@fs@iftopcapt\iftrue
}
\makeatother
\floatstyle{booktabsruled}
\restylefloat{algorithm}
\usepackage{amsthm}
\newtheoremstyle{theorem-style}
  {\topsep} 
  {\topsep} 
  {\itshape} 
  {} 
  {\bfseries} 
  {} 
  {\newline} 
  {} 
\theoremstyle{theorem-style}
\newtheorem{theorem}{Theorem}
\newtheorem{proposition}{Proposition}
\newtheorem{corollary}{Corollary}
\newtheorem{lemma}{Lemma}

\newtheoremstyle{definition-style}
  {\topsep} 
  {\topsep} 
  {} 
  {} 
  {\bfseries} 
  {} 
  {\newline} 
  {} 
\theoremstyle{definition-style}
\newtheorem{definition}{Definition}
\newtheorem{remark}{Remark}
\newtheorem{example}{Example}

\theoremstyle{theorem-style}
\newtheorem{innercustomthm}{Theorem}
\newenvironment{customthm}[1]
  {\renewcommand\theinnercustomthm{#1}\innercustomthm}
  {\endinnercustomthm}
\newtheorem{innercustomprop}{Proposition}
\newenvironment{customprop}[1]
  {\renewcommand\theinnercustomprop{#1}\innercustomprop}
  {\endinnercustomprop}

\theoremstyle{definition-style}
\newtheorem{innercustomdef}{Definition}
\newenvironment{customdef}[1]
  {\renewcommand\theinnercustomdef{#1}\innercustomdef}
  {\endinnercustomdef}

\newlength\figH
\newlength\figW
\usepackage[prependcaption,textsize=small,color=gray!40]{todonotes} 

\makeatletter
\newcommand{\superimpose}[2]{
  {\ooalign{$#1\@firstoftwo#2$\cr\hfil$#1\@secondoftwo#2$\hfil\cr}}}
\makeatother
\newcommand{\ostimes}{\mathbin{\mathpalette\superimpose{{\otimes}{\ominus}}}}

\makeatother

\DeclareSymbolFont{stmry}{U}{stmry}{m}{n}
\DeclareMathSymbol\leftarrowtriangle\mathrel{stmry}{"5E}
\DeclareMathSymbol\rightarrowtriangle\mathrel{stmry}{"5F}
\DeclareMathSymbol\sslash\mathrel{stmry}{"0C}
\DeclareMathSymbol\obar\mathrel{stmry}{"3A}
\DeclareMathSymbol\otimes\mathrel{stmry}{"0F}
\DeclareMathSymbol\ominus\mathrel{stmry}{"17}
\DeclareMathSymbol\minuso\mathrel{stmry}{"0A}

\renewcommand{\vec}{\operatorname{vec}}
\newcommand{\svec}{\operatorname{svec}}
\newcommand{\smat}{\operatorname{smat}}
\newcommand{\mat}{\operatorname{mat}}

\usepackage{url}
\usepackage{xr}
\usepackage{hyperref}
\hypersetup{
    colorlinks,
    linkcolor={black},
    citecolor={blue!50!black},
    urlcolor={blue!50!black}
}
\usepackage{cleveref}
\usepackage{bookmark} 

%% file: math_commands.tex
\usepackage{amsmath,amsfonts,bm}

\def\1{\bm{1}}

\def\rvx{{\bm{\mathsf{x}}}}

\def\rmA{{\bm{\mathsf{A}}}}

\def\rmH{{\bm{\mathsf{H}}}}

\def\rmL{{\bm{\mathsf{L}}}}

\def\rmX{{\bm{\mathsf{X}}}}

\def\vmu{{\bm{\mu}}}

\def\vb{{\bm{b}}}

\def\ve{{\bm{e}}}
\def\vf{{\bm{f}}}

\def\vk{{\bm{k}}}

\def\vr{{\bm{r}}}
\def\vs{{\bm{s}}}

\def\vu{{\bm{u}}}
\def\vv{{\bm{v}}}

\def\vx{{\bm{x}}}
\def\vy{{\bm{y}}}
\def\vz{{\bm{z}}}

\def\evtheta{{\theta}}

\def\evf{{f}}

\def\mA{{\bm{A}}}
\def\mB{{\bm{B}}}
\def\mC{{\bm{C}}}
\def\mD{{\bm{D}}}
\def\mE{{\bm{E}}}
\def\mF{{\bm{F}}}
\def\mG{{\bm{G}}}
\def\mH{{\bm{H}}}
\def\mI{{\bm{I}}}

\def\mK{{\bm{K}}}
\def\mL{{\bm{L}}}
\def\mM{{\bm{M}}}

\def\mP{{\bm{P}}}
\def\mQ{{\bm{Q}}}

\def\mS{{\bm{S}}}
\def\mT{{\bm{T}}}
\def\mU{{\bm{U}}}
\def\mV{{\bm{V}}}
\def\mW{{\bm{W}}}
\def\mX{{\bm{X}}}
\def\mY{{\bm{Y}}}
\def\mZ{{\bm{Z}}}

\def\mPi{{\bm{\Pi}}}
\def\mPhi{{\bm{\Phi}}}
\def\mPsi{{\bm{\Psi}}}
\def\mLambda{{\bm{\Lambda}}}
\def\mSigma{{\bm{\Sigma}}}
\def\mDelta{{\bm{\Delta}}}

\DeclareMathAlphabet{\mathsfit}{\encodingdefault}{\sfdefault}{m}{sl}
\SetMathAlphabet{\mathsfit}{bold}{\encodingdefault}{\sfdefault}{bx}{n}

\def\emX{{X}}

\newcommand{\R}{\mathbb{R}}

\DeclareMathOperator*{\argmin}{arg\,min}
\renewcommand{\top}{{\intercal}}